\documentclass{article}

\usepackage{microtype}
\usepackage{graphicx}
\usepackage{subcaption}
\usepackage{booktabs} 

\usepackage[preprint]{neurips_2026}

\usepackage[utf8]{inputenc} 
\usepackage[T1]{fontenc}    
\usepackage{hyperref}       
\usepackage{url}            
\usepackage{booktabs}       
\usepackage{amsfonts}       
\usepackage{nicefrac}       
\usepackage{microtype}      
\usepackage{xcolor}         
\usepackage{amsmath}
\usepackage{amssymb}
\usepackage{mathtools}
\usepackage{amsthm}
\usepackage{multirow}
\usepackage{graphicx}
\usepackage{tcolorbox}
\usepackage{listings}
\usepackage{enumitem}
\usepackage{wrapfig}
\usepackage{algorithm}
\usepackage{algorithmic}
\usepackage{tcolorbox}
\usepackage{textcomp}
\newcommand{\pyfence}{\texttt{\detokenize{```python}}}
\newcommand{\codefence}{\texttt{\detokenize{```}}}

\usepackage[capitalize,noabbrev]{cleveref}
\crefname{algorithm}{Algorithm}{Algorithms}
\Crefname{algorithm}{Algorithm}{Algorithms}

\theoremstyle{plain}
\newtheorem{theorem}{Theorem}[section]
\newtheorem{proposition}[theorem]{Proposition}

\newtheorem{corollary}[theorem]{Corollary}
\theoremstyle{definition}
\newtheorem{definition}[theorem]{Definition}
\newtheorem{assumption}[theorem]{Assumption}
\theoremstyle{remark}

\definecolor{darkgreen}{rgb}{0.0, 0.5, 0.0}

\newcommand{\greencirc}[1]{\textcolor{darkgreen}{\textcircled{\raisebox{-0.5pt}{\textbf{#1}}}}}
\newcommand{\redcirc}[1]{\textcolor{red}{\textcircled{\raisebox{-0.5pt}{\textbf{#1}}}}}

\title{Do Not Waste Your Rollouts: Recycling Search Experience for Efficient Test-Time Scaling}
\begin{document}

\renewcommand{\qedsymbol}{}
%

\renewcommand{\thefootnote}{\fnsymbol{footnote}}

\author{
    \textbf{Xinglin Wang}\textsuperscript{\rm 1}\footnotemark[1], \hspace{0cm}
    \textbf{Jiayi Shi}\textsuperscript{\rm 1}\footnotemark[1], \hspace{0cm}
    \textbf{Shaoxiong Feng}\textsuperscript{\rm 2}\footnotemark[2], \hspace{0cm}
    \textbf{Peiwen Yuan}\textsuperscript{\rm 1}, \hspace{0cm}
    \textbf{Yiwei Li}\textsuperscript{\rm 1}, \hspace{0cm}
    \textbf{Yueqi Zhang}\textsuperscript{\rm 1}, \hspace{0cm} \\
    \textbf{Chuyi Tan}\textsuperscript{\rm 1}, \hspace{0cm}
    \textbf{Ji Zhang}\textsuperscript{\rm 1}, \hspace{0cm}
    \textbf{Boyuan Pan}\textsuperscript{\rm 2}, \hspace{0cm}
    \textbf{Yao Hu}\textsuperscript{\rm 2}, \hspace{0cm}
    \textbf{Kan Li}\textsuperscript{\rm 1}\footnotemark[2] \\
    \textsuperscript{\rm 1} School of Computer Science, Beijing Institute of Technology \\
    \textsuperscript{\rm 2} Xiaohongshu Inc \\
    \texttt{\{wangxinglin,shijiayi,peiwenyuan,liyiwei\}@bit.edu.cn} \\
    \texttt{\{zhangyq,tanchuyi,zhangji,likan\}@bit.edu.cn} \\
    \texttt{shaoxiongfeng2023@gmail.com} \quad
    \texttt{\{panboyuan,xiahou\}@xiaohongshu.com}
}

\footnotetext[1]{Equal contribution.}
\footnotetext[2]{Corresponding author.}

\renewcommand{\thefootnote}{\arabic{footnote}}

\maketitle

\begin{abstract}

Test-Time Scaling enhances the reasoning capabilities of Large Language Models by allocating additional inference compute to broaden the exploration of the solution space. However, existing search strategies typically treat rollouts as disposable samples, where valuable intermediate insights are effectively discarded after each trial. This wasted rollout-level experience leads to substantial computational redundancy, as models repeatedly re-derive discovered conclusions and revisit known dead ends across extensive attempts. To bridge this gap, we propose \textbf{Recycling Search Experience (RSE)}, a self-guided, training-free strategy that turns test-time search from a series of isolated trials into a cumulative, experience-guided process. By actively distilling raw trajectories into a shared experience bank, RSE enables positive recycling of intermediate conclusions to shortcut redundant derivations and negative recycling of failure patterns to prune encountered dead ends. Theoretically, we provide an analysis that formalizes the efficiency gains of RSE over independent sampling in solving complex reasoning tasks. Empirically, extensive experiments on HMMT24, HMMT25, IMO-Bench, and HLE show that RSE consistently outperforms strong baselines under comparable computational budgets, establishing a strong compute-efficiency frontier for test-time scaling.\footnote{Our code and data have been released on \url{https://github.com/WangXinglin/RSE}.}

\end{abstract}



\section{Introduction}

\begin{flushright}
\begin{minipage}{0.72\linewidth}
\raggedleft
\small\itshape
``Wisdom comes from experience. Experience is often a result of lack of wisdom.''\\[-0.2em]
\small\normalfont --- Terry Pratchett
\end{minipage}
\end{flushright}
\vspace{-0.8em}


Scaling parameters and data during pre-training has long been the primary driver of LLM performance, yet recent breakthroughs highlight a paradigm shift towards Test-Time Scaling (TTS) \citep{brown2024large}. By leveraging extended computation at test time, models can effectively tackle complex reasoning tasks that were previously out of reach \citep{o1, DeepSeek-R1}.
TTS approaches can be broadly divided into two categories: \emph{Internal} TTS, which trains the LLMs
to “think” slowly with long Chain-of-Thought (CoT)~\citep{QwQ, Kimi-k1.5}, and \emph{External} TTS, which improves performance by allocating additional inference-time compute with a fixed LLM~\citep{snell2024scaling, wu2025inference}. 
Focusing on the latter, test-time search has emerged as a prominent paradigm, where the model samples extensive rollouts to broaden the exploration of the solution space \citep{zhang2025survey, liu2025can}.





\begin{figure*}[th]
\begin{center}
\includegraphics[width=1.0\textwidth]{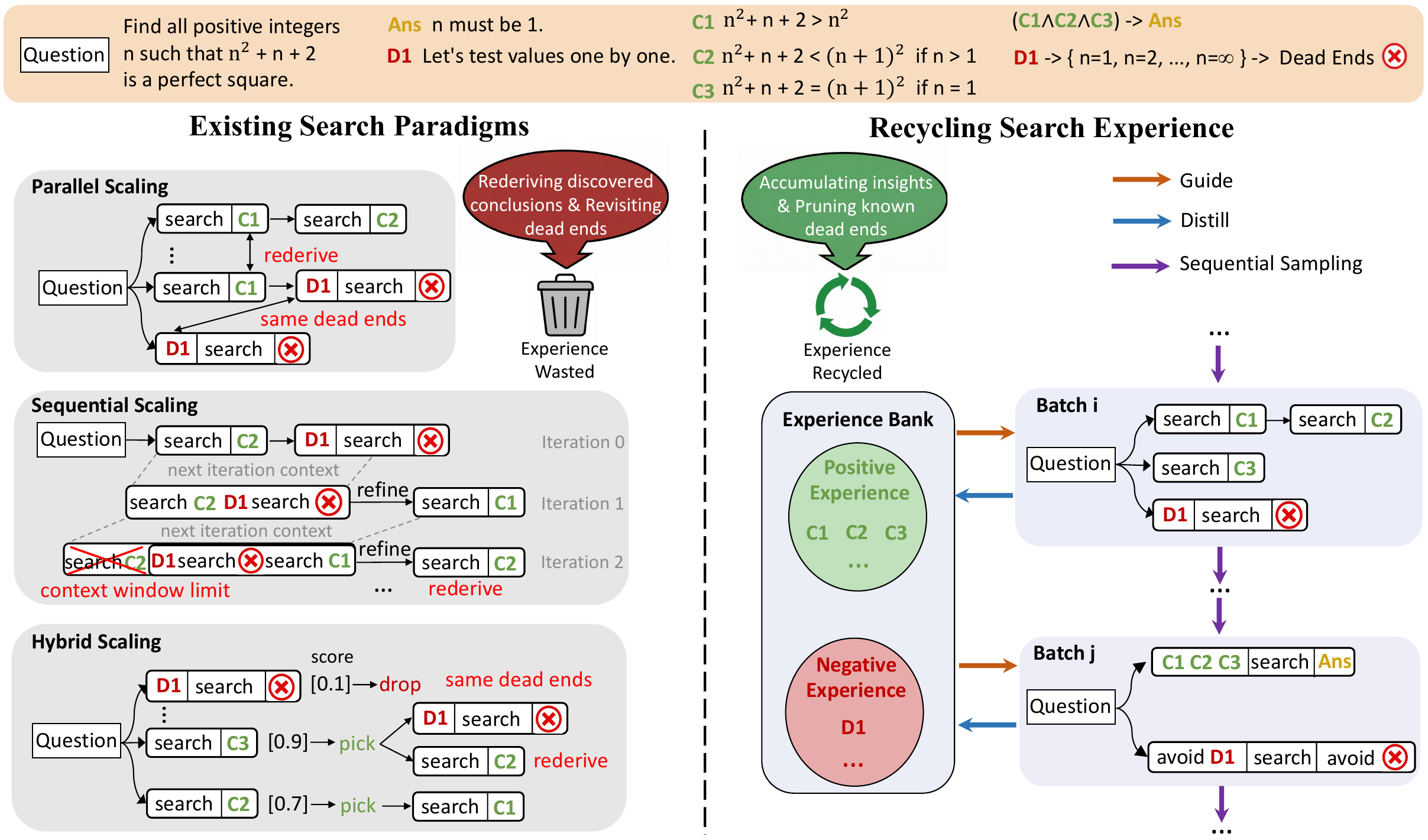}
\end{center}
\caption{\textbf{From disposable rollouts to experience-guided search.} 
\textbf{(Left)} Existing test-time scaling paradigms (parallel, sequential, and hybrid) largely treat rollouts as disposable: intermediate conclusions are repeatedly re-derived and dead ends are revisited across rollouts. 
\textbf{(Right)} Recycling Search Experience (RSE) runs rollouts in batches, distills reusable trajectory information into a shared Experience Bank, and conditions subsequent exploration on it. Positive recycling shares \textbf{intermediate conclusions} (e.g., $\mathbf{C_1}$--$\mathbf{C_3}$) to shortcut redundant derivations, while negative recycling records \textbf{failure patterns} (e.g., $\mathbf{D_1}$) to prune known dead ends.}
\label{fig:overview}
\end{figure*}


While yielding promising performance gains, existing test-time search strategies share a common bottleneck in information utilization: rollouts are often treated as \textit{disposable attempts} rather than \textit{experience to be distilled and reused} (see Figure~\ref{fig:overview}).
Specifically, \textit{parallel scaling} expands breadth via extensive independent rollouts, but typically offers limited cross-branch sharing of intermediate insights \citep{self-consistency, lightman2023let, ESC, wang2025make, zheng2026parallel}; \textit{sequential scaling} iteratively improves a single draft, yet the information it accumulates is confined to unstructured in-context history: finite context windows force truncation or compression of earlier revisions, and the search remains a path-dependent local refinement around the current draft rather than a reusable pool of experience \citep{Self-Refine, shinn2023reflexion, chenteaching}; \textit{hybrid scaling} further introduces tree-structured search control over candidate prefixes, using external process reward models \citep{snell2024scaling, beeching2024scaling, wu2025inference, liu2025can, Wang2025EveryRC} or look-ahead evaluation \citep{yao2023tree, wan2024alphazero, park2025ensembling} to estimate prefix quality and allocate rollout budgets, expanding promising branches while pruning low-value ones.
However, the information retained by such search control is mainly branch-local prefixes and their associated value signals, rather than what has been discovered along completed trajectories or pruned paths.
Intermediate conclusions, useful facts, and failure causes discovered in one branch are rarely distilled into reusable guidance for other branches or later rollouts. 
Despite these different forms of information reuse, existing test-time search methods still lack a systematic mechanism for accumulating rollout-level experience across the overall search process.
As a result, models can repeatedly re-derive insights already found in other rollouts and revisit similar dead ends, leading to substantial redundant computation.

To address this inefficiency, we propose \textbf{Recycling Search Experience (RSE)}, a self-guided, training-free strategy that turns rollouts from disposable samples into reusable search experience to actively guide the search process. Instead of viewing rollouts as isolated trials, we treat search as a cumulative process, where valuable insights from prior trajectories are fed back to guide later exploration. Leveraging the model's self-assessment capability \citep{weng2023large,dhuliawala2024chain}, RSE distills valuable insights from rollouts into a shared experience bank without relying on external supervision, and explicitly conditions subsequent exploration on this bank. This mechanism enables two forms of experience reuse: \textit{positive recycling}, where intermediate conclusions are shared to shortcut redundant derivations, and \textit{negative recycling}, where encountered dead ends serve as constraints to prune the search space. By iterating this distill-and-guide loop, RSE turns test-time scaling from isolated trials into a cumulative, experience-guided search that concentrates compute on promising, unexplored regions of the solution space.



Intuitively, solving complex reasoning problems via independent sampling requires at least one rollout to independently derive a complete chain of necessary intermediate conclusions, whose probability can decay exponentially with problem complexity. In contrast, RSE effectively ``checkpoints'' valid intermediate conclusions, allowing later rollouts to build on fragmented partial successes rather than starting from scratch. We provide a theoretical analysis that formalizes and supports this intuition, including a bounded-degradation result for imperfect experience extraction (see Appendix~\ref{sec:theory}).



To evaluate the effectiveness of RSE, we conduct extensive experiments on challenging mathematical reasoning benchmarks, including HMMT24, HMMT25~\citep{balunovic_srimatharena_2025}, IMO-Bench~\citep{luong2025towards}, and HLE-Math-text~\citep{phan2025humanity}, across a broad range of rollout budgets and model families. We further evaluate RSE beyond mathematical reasoning on coding, expert-level QA, and long-horizon planning tasks. The empirical results show that RSE consistently outperforms strong baseline strategies under comparable computational budgets, highlighting both its compute efficiency and cross-domain generalizability.

To summarize, our main contributions are:

\begin{itemize}
\item We propose Recycling Search Experience (RSE), a training-free inference strategy that turns test-time search into a cumulative, experience-guided process by recycling rollout-level experience.
\item We provide a theoretical analysis that formalizes the sample-efficiency advantage of experience recycling over independent sampling, with a bounded-degradation result under imperfect experience extraction.
\item We empirically evaluate RSE on challenging mathematical reasoning benchmarks, demonstrating consistent gains over strong baselines across diverse models and rollout budgets, and further show its effectiveness beyond mathematical reasoning on coding, expert-level QA, and long-horizon planning tasks.
\end{itemize}

\section{Related Work}

\paragraph{Test-Time Scaling.}
The paradigm of scaling test-time compute has emerged as a critical avenue for enhancing reasoning capabilities~\citep{o1, DeepSeek-R1}.
Early approaches explored two primary dimensions: scaling depth via sequential refinement~\citep{wei2022chain, Self-Refine, shinn2023reflexion}, which iteratively improves a single chain of thought; and scaling width via ensemble sampling~\citep{self-consistency}, which leverages the diversity of independent rollouts to marginalize out errors.
To add structure beyond depth/width scaling, prior work explores tree-based search such as Tree of Thoughts~\citep{yao2023tree} and MCTS~\citep{liu2023don, wan2024alphazero, inouewider, li2025rethinkmcts}, which build a search tree over candidate prefixes and use lookahead evaluation to estimate node quality and guide branch expansion. More recently, Process Reward Models (PRMs) provide step-level feedback for pruning and budget allocation, reducing the reliance on expensive lookahead simulation at inference~\citep{lightman2023let, snell2024scaling, beeching2024scaling, wu2025inference}. Together, these methods make test-time search selective rather than exhaustive, using scores or value estimates of intermediate prefixes to prioritize promising branches, prune low-value ones, and allocate compute more adaptively.
However, these signals are mainly used to control where the search proceeds, deciding which prefixes to expand, prune, or allocate additional compute to, rather than preserving the reusable experience discovered in previous rollouts.
In contrast, RSE focuses on self-guided experience recycling: it distills completed rollouts into positive and negative textual experience and reuses this experience across later rollouts, without requiring ground-truth labels or external evaluators.

Within this line of research, the most relevant concurrent study is PaCoRe~\citep{pacore2025}, which similarly integrates search history into the context. PaCoRe concatenates final answers from prior rollouts for consistency calibration, while discarding the intermediate thinking content, which could be mined as reusable experience. In contrast, RSE actively mines reusable experience from the entire exploration trajectory to guide subsequent exploration.

\paragraph{Agentic Memory.}
Memory mechanisms empower agents to persist information beyond limited context windows~\citep{packer2023memgpt, park2023generative, hu2025memory, zhang2026praise}. In the reasoning domain, existing methods primarily utilize memory for cross-query transfer, accumulating experience across different tasks. Prominent approaches include building skill libraries for code generation~\citep{wang2023voyager}, retrieving relevant trajectories from historical databases~\citep{ouyang2025reasoningbank}, or distilling domain-general heuristics via semantic group advantages~\citep{cai2025training} to aid similar future queries. Similarly, Batch-of-Thought~\citep{yang2026batch} facilitates cross-instance learning within a batch of questions. Distinct from these approaches that seek generalizable patterns from external history or parallel instances, our work targets intra-query memory reuse. RSE specifically recycles the instance-specific exploration experience generated within the current search process to guide subsequent exploration and avoid redundant errors.


\paragraph{LLM Context Engineering.}
Context engineering studies how to systematically construct, process, and manage the information payload provided to LLMs under a limited context window, beyond prompt wording alone~\citep{mei2025survey}.
A common direction is to improve information density by selecting or compressing long inputs while preserving task-relevant semantics, e.g., LLMLingua~\citep{jiang2023llmlingua, pan2024llmlingua} and Selective Context~\citep{li2023compressing, Yuan2025KardiaR1UL}, as well as long-horizon context summarization for agent settings~\citep{wu2025resum, kang2025acon}.
RSE can be viewed through the same lens in a search setting: instead of retaining verbose, low-signal rollout traces, RSE distills the search process into compact, structured experience and maintains it via deduplication, maximizing the utility of prior search experience within a fixed context budget.

\section{Methodology}
\label{sec:method}

Recycling Search Experience (RSE) is motivated by a simple principle: test-time search should not discard the valuable information contained in rollouts. We therefore treat search as a cumulative process, where rollouts from earlier rounds are distilled into reusable experience and fed back to guide later exploration under the same compute budget.
The overall workflow of RSE is summarized in Algorithm \ref{alg:rse}, comprising three coordinated components: Batched Experience-Guided Search, Self-Guided Distillation, and Semantic Experience Deduplication.

\subsection{Batched Experience-Guided Search}
\label{sec:batched_search}
Standard parallel sampling ensures breadth but suffers from \textit{isolation}, where rollouts cannot share intermediate discoveries. Conversely, purely sequential refinement enables reuse but is hindered by limited parallelism and context limits, often forcing truncation or compression of earlier attempts and keeping exploration local around the current draft. A natural question is \textit{whether we can share experience during a fully parallel rollout}. In practice, mid-trajectory synchronization is difficult to define and implement robustly:
evaluating partial reasoning states is unreliable (e.g., distinguishing an actual error from a complex intermediate step) \citep{pandit2025hard2verify}, and injecting global updates changes the conditioning context, which can disrupt the current reasoning trajectory and increase output variance \citep{laban2025llms}.

These issues make real-time experience sharing impractical without specific architectural changes.
RSE therefore adopts a \textbf{batched experience-guided} protocol that offers a simple and stable coordination interface. We partition the rollout budget into $R$ rounds; in each round $r$, the model generates a batch of $K_r$ trajectories in parallel. Across rounds, the system maintains a global \textbf{Experience Bank} 
$(\mathcal{E}_{r}^{pos}, \mathcal{E}_{r}^{neg})$, which stores the accumulated experience entries from previous rounds. 
Before launching round $r$, we serialize this bank into the prompt (see Appendix~\ref{app:prompts} for details), so all rollouts in the batch start from a synchronized state. This design preserves within-round diversity while enabling cross-round information reuse, thereby balancing information reuse with practical inference latency.

\subsection{Self-Guided Experience Distillation}
\label{sec:distill}

Simply concatenating all previous trajectories into the context window is impractical due to length constraints and the low signal-to-noise ratio of raw reasoning chains. We leverage the model's intrinsic capability to critique its own reasoning and identify valuable experience without external supervision, avoiding the need for ground-truth labels or external reward models, thereby making RSE applicable to general search environments.

Based on this, we propose \textbf{Self-Guided Experience Distillation}.
Instead of retaining entire trajectories, we employ a lightweight prompting step to distill each rollout $\omega$ into discrete, structured experience items (see Appendix~\ref{app:prompts} for details). Specifically, for every generated trajectory, the model extracts:
(1) \textbf{Positive Experience ($\mathcal{E}^{pos}$):} Verified propositions or lemmas that serve as ``truth anchors'' for future batches;
(2) \textbf{Negative Experience ($\mathcal{E}^{neg}$):} Critical pitfalls or strategic dead ends that act as ``negative constraints.''
This explicit extraction transforms unstructured exploration logs into compact, actionable guidance, facilitating high-quality experience reuse for future exploration.

\begin{wrapfigure}[26]{r}{0.575\textwidth}
\vspace{-0.8em}
\hrule height 0.8pt
\vspace{2pt}
\refstepcounter{algorithm}
\label{alg:rse}
\noindent\textbf{Algorithm~\thealgorithm} Recycling Search Experience (RSE)
\vspace{2pt}
\hrule height 0.4pt
\vspace{3pt}

\small
\begin{algorithmic}[1]
\REQUIRE Problem $x$, model $\pi$, rounds $R$, batch sizes $\{K_r\}_{r=1}^R$, similarity threshold $\tau$.
\ENSURE Final batch of trajectories $\Omega_R$.
\STATE Initialize Experience Bank: $\mathcal{E}_0^{pos}\gets\emptyset$, $\mathcal{E}_0^{neg}\gets\emptyset$.
\FOR{$r = 1$ \textbf{to} $R$}
    \STATE \textcolor{gray}{\small \textbf{// Step 1: Batched Experience-Guided Search (Sec.~\ref{sec:batched_search})}}
    \STATE $u_r \gets \mathsf{Prompt}(x,(\mathcal{E}_{r-1}^{pos},\mathcal{E}_{r-1}^{neg}))$.
    \STATE Sample in parallel: $\Omega_r \gets \{\omega_r^{(i)} \sim \pi(\cdot \mid u_r)\}_{i=1}^{K_r}$.

    \STATE \textcolor{gray}{\small \textbf{// Step 2: Self-Guided Experience Distillation (Sec.~\ref{sec:distill})}}
    \STATE Initialize batch experience: $\Delta_r^{pos}\gets\emptyset$, $\Delta_r^{neg}\gets\emptyset$.
    \FOR{\textbf{each} rollout $\omega \in \Omega_r$}
        \STATE $(\delta^{pos}, \delta^{neg}) \gets \mathsf{Distill}(x,\omega)$ 
        \STATE $\Delta_r^{pos} \gets \Delta_r^{pos} \cup \delta^{pos}$; \quad $\Delta_r^{neg} \gets \Delta_r^{neg} \cup \delta^{neg}$.
    \ENDFOR

    \STATE \textcolor{gray}{\small \textbf{// Step 3: Semantic Experience Deduplication (Sec.~\ref{sec:dedup})}}
    \FOR{type $\in \{pos, neg\}$}
        \STATE $\mathcal{E}_r^{\text{type}} \gets \mathcal{E}_{r-1}^{\text{type}}$
        \FOR{\textbf{each} $\delta \in \Delta_r^{\text{type}}$}
            \IF{$\mathcal{E}_r^{\text{type}}=\emptyset$ \textbf{or} $\max_{e \in \mathcal{E}_r^{\text{type}}} \mathsf{Sim}(\delta, e) < \tau$}
                \STATE $\mathcal{E}_r^{\text{type}} \gets \mathcal{E}_r^{\text{type}} \cup \{\delta\}$
            \ENDIF
        \ENDFOR
    \ENDFOR
\ENDFOR
\STATE \textbf{return} $\Omega_R$.
\end{algorithmic}

\vspace{2pt}
\hrule height 0.4pt
\vspace{-0.8em}
\end{wrapfigure}

\subsection{Semantic Experience Deduplication}
\label{sec:dedup}

Even with distilled items, the accumulated experience can rapidly grow over multiple rounds, risking context overflow. Moreover, parallel rollouts within a batch often exhibit high redundancy: simple steps or common errors tend to be discovered repeatedly by multiple trajectories. Without mitigation, these repetitive items could dominate the prompt, crowding out rarer, high-value insights (Table~\ref{tab:ablation_tau_combined}).

To address this, we employ a \textbf{Semantic Experience Deduplication} strategy to maintain the Experience Bank as a diverse set rather than a simple list.
As detailed in Algorithm \ref{alg:rse}, we employ an incremental greedy selection strategy. By evaluating candidates against the dynamically updated Experience Bank, we ensure that each new entry is semantically distinct from all previously admitted items, effectively filtering out redundancy both from historical rounds and within the current batch.
This mechanism prevents context explosion by filtering out repetitive experiences, thereby maintaining high information density within the limited context window.
\section{Experiments}
\label{sec:experiments}

In this section, we present the experimental evaluation of RSE. We begin by outlining the experimental setup, followed by the main results comparing RSE against baselines, and an in-depth analysis focusing on scalability, efficiency, reasoning dynamic and the impact of experience context construction strategies. Additionally, a concrete case analysis is provided in Appendix~\ref{app:rse_case}.

\newcommand{\g}[1]{\textcolor{gray}{#1}}
\newcommand{\na}{--} 

\subsection{Experimental Setup}

\paragraph{Benchmarks and Models.}
We primarily evaluate the proposed RSE strategy on mathematical reasoning with four challenging benchmarks: HMMT24, HMMT25, IMO-AnswerBench, and a 100-sample math subset from Humanity's Last Exam (HLE-Math-text). These datasets serve as reliable proxies for advanced problem-solving capabilities requiring complex multi-step reasoning. 
We evaluate RSE primarily on models specialized for complex reasoning tasks, spanning diverse scales and architectures: \textsc{Qwen3-30B-A3B-Thinking-2507}, \textsc{Qwen3-4B-Thinking-2507} \citep{llm-qwen3}, \textsc{Phi-4-reasoning} \citep{llm-phi4}, and \textsc{Deepseek-V3.2} \citep{llm-deepseek}\footnote{The results of \textsc{Deepseek-V3.2} are detailed in Appendix~\ref{app:table-ds-v3.2}.}. 
To further verify the cross-task generalizability of RSE, we extend our empirical evaluation to three additional tasks: coding (LiveCodeBench-v6 \citep{jainlivecodebenchcus}), expert-level general reasoning (GPQA-Diamond \citep{rein2024gpqa}), and complex long-horizon planning (TravelPlanner \citep{xie2024travelplanner}).
We additionally extend our analysis to general-purpose instruction-tuned models to demonstrate the universality of RSE across training paradigms.

\begin{table*}[t]
    \centering
    \small
    \renewcommand{\arraystretch}{0.85}
    \caption{Performance comparison across different models and iterations on mathematical reasoning benchmarks. Values are reported as Pass@1 (\%). ``It0'' denotes the base/initial performance, while ``It1-3'' represent subsequent iterative refinements. Gray values in It0 indicate the Base performance carried over for comparison. Best performance in each iteration is \textbf{bolded}.}
    \setlength{\tabcolsep}{3.2pt}
    \begin{tabular}{lcccccccccccccccc}
        \toprule
        \multirow{2}{*}{\textbf{Method}} & \multicolumn{4}{c}{\textbf{HMMT24}} & \multicolumn{4}{c}{\textbf{HMMT25}} & \multicolumn{4}{c}{\textbf{IMO-AnswerBench}} & \multicolumn{4}{c}{\textbf{HLE-Math-text}} \\
        \cmidrule(lr){2-5} \cmidrule(lr){6-9} \cmidrule(lr){10-13} \cmidrule(lr){14-17}
         & It0 & It1 & It2 & It3 & It0 & It1 & It2 & It3 & It0 & It1 & It2 & It3 & It0 & It1 & It2 & It3 \\
        \midrule
        \multicolumn{17}{c}{\textit{\textbf{Qwen3-30B-A3B-Thinking-2507}}} \\
        \midrule
        Base & 57.4 & \na & \na & \na & 69.0 & \na & \na & \na & 50.5 & \na & \na & \na & 24.0 & \na & \na & \na \\
        MV@128 & 68.9 & \na & \na & \na & 74.8 & \na & \na & \na & 56.3 & \na & \na & \na & 33.4 & \na & \na & \na \\
        Self-Ref & \g{57.4} & 62.7 & 64.5 & 65.6 & \g{69.0} & 71.4 & 72.5 & 72.6 & \g{50.5} & 49.5 & 49.9 & 50.1 & \g{24.0} & 26.1 & 26.0 & 26.2 \\
        PaCoRe & \g{57.4} & 70.6 & 72.2 & 73.1 & \g{69.0} & 78.4 & 79.7 & 80.2 & \g{50.5} & 57.3 & 57.8 & 58.0 & \g{24.0} & 38.0 & 40.2 & 40.7 \\
        RSE & \g{57.4} & \textbf{71.9} & \textbf{73.7} & \textbf{74.4} & \g{69.0} & \textbf{81.3} & \textbf{82.9} & \textbf{83.9} & \g{50.5} & \textbf{59.2} & \textbf{60.1} & \textbf{60.3} & \g{24.0} & \textbf{39.9} & \textbf{43.0} & \textbf{44.8} \\
        \midrule
        \multicolumn{17}{c}{\textit{\textbf{Qwen3-4B-Thinking-2507}}} \\
        \midrule
        Base & 42.6 & \na & \na & \na & 54.0 & \na & \na & \na & 42.1 & \na & \na & \na & 14.2 & \na & \na & \na \\
        MV@128 & 53.5 & \na & \na & \na & 61.8 & \na & \na & \na & 48.5 & \na & \na & \na & 17.1 & \na & \na & \na \\
        Self-Ref & \g{42.6} & 47.0 & 48.8 & 49.7 & \g{54.0} & 58.2 & 58.6 & 59.4 & \g{42.1} & 43.5 & 46.2 & 46.3 & \g{14.2} & 14.8 & 14.8 & 15.0 \\
        PaCoRe & \g{42.6} & \textbf{56.7} & 54.6 & 54.9 & \g{54.0} & \textbf{69.6} & 70.5 & 70.0 & \g{42.1} & 46.0 & 48.2 & 49.0 & \g{14.2} & \textbf{19.6} & 19.4 & 19.5 \\
        RSE & \g{42.6} & 54.3 & \textbf{55.1} & \textbf{56.0} & \g{54.0} & 68.6 & \textbf{72.2} & \textbf{73.5} & \g{42.1} & \textbf{48.6} & \textbf{49.6} & \textbf{49.6} & \g{14.2} & 19.5 & \textbf{20.2} & \textbf{20.8} \\
        \midrule
        \multicolumn{17}{c}{\textit{\textbf{Phi-4-Reasoning}}} \\
        \midrule
        Base & 40.3 & \na & \na & \na & 43.9 & \na & \na & \na & 34.5 & \na & \na & \na & 8.5 & \na & \na & \na \\
        MV@128 & 53.3 & \na & \na & \na & 56.7 & \na & \na & \na & 46.0 & \na & \na & \na & 9.0 & \na & \na & \na \\
        Self-Ref & \g{40.3} & 41.4 & 42.3 & 43.0 & \g{43.9} & 46.8 & 48.5 & 49.9 & \g{34.5} & 35.5 & 35.1 & 34.7 & \g{8.5} & 8.1 & 8.0 & 8.0 \\
        PaCoRe & \g{40.3} & 49.0 & 52.4 & 52.3 & \g{43.9} & 59.0 & 62.0 & 63.8 & \g{34.5} & 39.2 & 40.1 & 40.6 & \g{8.5} & 8.1 & 9.3 & 9.6 \\
        RSE & \g{40.3} & \textbf{51.0} & \textbf{55.9} & \textbf{56.5} & \g{43.9} & \textbf{60.2} & \textbf{66.0} & \textbf{67.5} & \g{34.5} & \textbf{40.5} & \textbf{40.3} & \textbf{42.3} & \g{8.5} & \textbf{9.0} & \textbf{10.9} & \textbf{11.4} \\
        \bottomrule
    \end{tabular}%
    \label{tab:benchmark_results}
\end{table*}

\paragraph{Baselines.} We compare RSE against three distinct categories of inference strategies: 
(1) \textbf{Standard Sampling}, evaluating the model's intrinsic performance; 
(2) \textbf{Majority Voting} \citep{self-consistency}, mitigating stochasticity by aggregating consensus across multiple reasoning paths;
(3) \textbf{Self-Refine} \citep{Self-Refine}, performing iterative refinement on a single reasoning trajectory; and
(4) \textbf{PaCoRe} \citep{pacore2025}, embedding historical information via direct concatenation of past outputs.
we exclude structured search methods such as Tree-of-Thought, MCTS-style search, and PRM-guided search, as they typically introduce additional external guidance.
Additionally, we provide a supplementary comparison with confidence-based strategies, specifically \textbf{Self-Certainty} \citep{baseline-self-c} and \textbf{Deep-Confidence} \citep{baseline-deep-c}. Detailed results for these methods are discussed in Appendix~\ref{app:add_baselines}.

\paragraph{Evaluation Protocol.}
Unless otherwise specified, we adhere to the following configurations to ensure rigorous comparison. 
For Standard Sampling, we conduct a single large-scale experiment consisting of $1,024$ independent stochastic rollouts to estimate the intrinsic pass@1 accuracy. 
For Majority Voting, results are derived from a budget of $128$ rollouts. 
For both RSE and Sequential Optimization Baselines, we standardize the process into a common multi-round search framework. By default, this consists of $1$ reference initialization iteration $3$ subsequent optimization iterations. 
Specifically, to ensure statistical reliability, we conduct $8$ independent runs for all experiments, each initialized with a population of $32$ distinct reference responses. In each iteration, the method generates $32$ new rollouts per run, which serve as reference responses for the next round. 
To monitor the optimization progress, we report the Pass@1 accuracy of the newly generated rollouts at each iteration, averaged over the $8$ independent runs.
Further implementation details are provided in Appendix~\ref{app:exp_details}. The full set of prompt templates is detailed in Appendix~\ref{app:prompts}.

\paragraph{Model-Specific Configurations.}
We evaluate \textsc{Deepseek-V3.2} on the 100-sample HLE-Math-text subset exclusively. 
For \textsc{Phi-4-reasoning}, we implement specific context truncation to accommodate PaCoRe within the model's 32k window limit. More details are provided in Appendix~\ref{app:model_adaptations}.

\subsection{Main Results}
\label{sec:main_results}

The main experimental results across four mathematical reasoning benchmarks and three model architectures are summarized in Table~\ref{tab:benchmark_results}. We observe consistent trends demonstrating the superiority of RSE over both standard sampling and sequential optimization baselines.

\paragraph{RSE Achieves Best Performance among Baselines.}
As shown in Table~\ref{tab:benchmark_results}, RSE consistently outperforms the selected training-free and self-guided baselines under the same multi-round search protocol.
Crucially, this superiority demonstrates robustness across both model scales and training paradigms. 
On smaller models like \textsc{Qwen3-4B-Thinking}, RSE unlocks capabilities that were previously inaccessible, enabling it to reach 73.5\% on HMMT25, competitively matching the baseline performance of the much larger 30B model (69.0\%). This indicates that effective test-time search can compensate for parametric limitations to a significant extent. 
Furthermore, extended evaluations on instruction-tuned models (Table~\ref{tab:instruct_results}) confirm that RSE generalizes beyond reasoning-specialized architectures. 
Moreover, we provide a significance analysis in Appendix~\ref{app:significance_analysis}, showing that the final-iteration improvements of RSE over the strongest iterative baselines are statistically reliable across repeated runs.

\paragraph{Scaling Dynamics: Higher Ceilings and Delayed Convergence.}

Beyond cumulative gains, RSE exhibits a more favorable empirical scaling trajectory than the compared baselines.
First, RSE reaches a higher final performance level than Self-Refine and PaCoRe within the tested iteration budget, indicating a higher empirical ceiling in our experiments.
Second, RSE shows delayed convergence compared with these sequential baselines.
While Self-Refine and PaCoRe typically exhibit early saturation, RSE continues to obtain gains through Iteration 3, suggesting that experience recycling can sustain the marginal utility of later search rounds.

\paragraph{Efficacy on Hard Reasoning Tasks.}

On challenging benchmarks like HLE-Math-text, while Majority Voting offers significant gains, RSE consistently achieves a much better performance. This divergence highlights a critical distinction: whereas MV remains constrained by the statistical dominance of the model's initial distribution, RSE effectively navigates the reasoning search space to extract valid solutions initially assigned low probabilities. By iteratively refining these paths, RSE can partially mitigate recurring reasoning errors that consensus-based methods may fail to correct.

\begin{table}[ht]
\small
\centering
\renewcommand{\arraystretch}{0.85}
\caption{Performance evaluation of RSE and baselines beyond mathematical reasoning. MV@128 is omitted for TravelPlanner as consensus voting is not directly applicable to open-ended planning tasks. All experiments are conducted with \textsc{Qwen3-30B-A3B-Thinking-2507}.}
\label{tab:generality_evaluation}
\begin{tabular}{llccccc}
\toprule
\textbf{Benchmark} & \textbf{Metric} & \textbf{Base} & \textbf{Self-Refine} & \textbf{PaCoRe} & \textbf{MV@128} & \textbf{RSE } \\
\midrule
LiveCodeBench-v6 & pass@1 & 61.19 & 64.13 & 60.26 & 64.63 & \textbf{68.55} \\
GPQA-Diamond & pass@1 & 69.17 & 70.77 & 69.91 & 70.72 & \textbf{72.67} \\
TravelPlanner & Final Pass Rate & 0.76 & 28.12 & 0.00 & -- & \textbf{34.44} \\
\bottomrule
\end{tabular}
\end{table}

\subsection{Generality Beyond Mathematical Reasoning}
\label{app:add_benchmarks}

As illustrated in Table \ref{tab:generality_evaluation}, RSE consistently outperforms all baselines across coding, expert-level general reasoning and complex long-horizon planning tasks. Notably, while Majority Voting (MV@128) achieves competitive results on closed-form QA tasks, it is inherently inapplicable to open-ended planning tasks like TravelPlanner due to the lack of a well-defined answer space for consensus aggregation. Furthermore, PaCoRe experiences a complete performance collapse on TravelPlanner (0.00 Final Pass Rate), which indicates that directly concatenating verbose planning trajectories fails to provide effective guidance and likely overwhelms the model's reasoning process. In contrast, RSE's structured extraction of intermediate insights and failure patterns proves highly effective across all domains, yielding a pass@1 of 68.55 on LiveCodeBench-v6 and a Final Pass Rate of 34.44 on TravelPlanner. These results suggest that RSE is not limited to mathematical reasoning and can generalize to broader reasoning domains.

\subsection{Analysis}
\label{sec:ablation}

We analyze the effectiveness of RSE from three key perspectives: (1) Scalability and Efficiency, assessing performance gains across search depth, width, and compute cost; (2) Reasoning Dynamics, investigating how experience recycling influences exploration behavior; and (3) Experience Composition, dissecting the impact of distinct experience types and selection strategies.
Additionally, Appendix~\ref{app:additional_analyis} provides supplementary analyses covering significance analysis, comparisons with confidence-based baselines, robustness and configuration sensitivity studies and quality analysis of distilled experiences.
Unless otherwise specified, all experiments in analysis experiments are conducted using \textsc{Qwen3-30B-A3B-Thinking-2507}.

\begin{figure*}[t]
    \centering
    \includegraphics[width=1.0\textwidth]{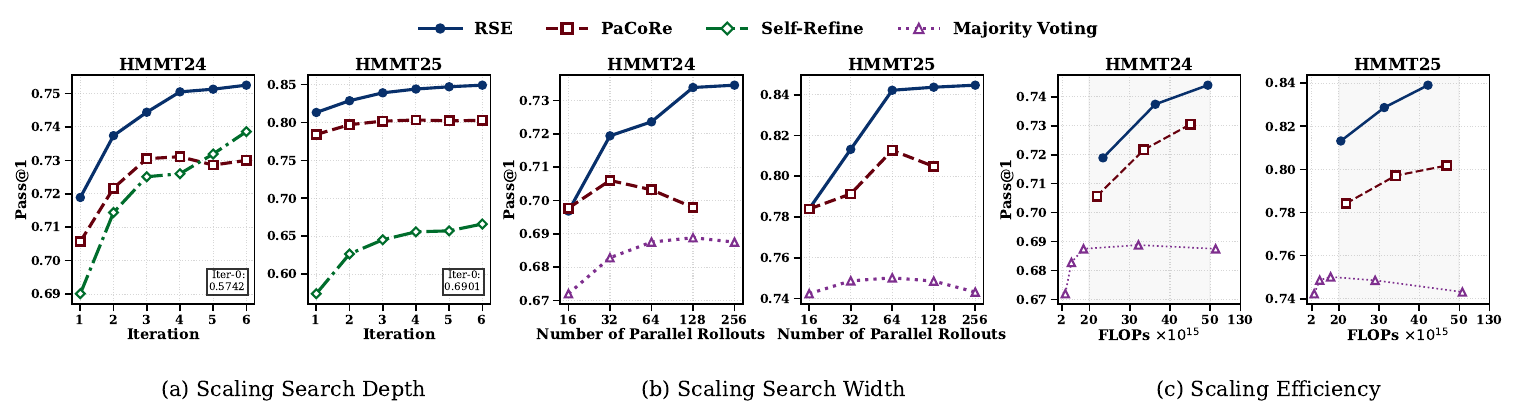}
    \caption{
    \textbf{Scalability and Efficiency Analysis of Test-Time Search.} We evaluate the scaling behaviors of different search strategies across three dimensions: (a) search depth, (b) search width, and (c) computational efficiency .
    }
    \label{fig:analysis_scaling}
\end{figure*}

\subsubsection{Scalability and Efficiency}

We systematically dissect the scalability of RSE across three fundamental dimensions: search depth, search width, and compute flops. The results are summarized in Figure~\ref{fig:analysis_scaling}. 
Furthermore, beyond structural scaling, RSE demonstrates robust scaling behavior across different configurations. 
It exhibits predictable and stable performance trends with respect to both experience bank capacity and per-iteration rollout counts, while also tolerating noise interference within the experience bank. Detailed analyses are provided in Appendix~\ref{sec:robustness_sensitivity}.

\paragraph{Scaling Search Depth.}
Figure~\ref{fig:analysis_scaling}(a) illustrates a distinct divergence in scaling behaviors (Majority Voting is excluded from this comparison, as its non-iterative nature lacks the mechanism to refine prior rollouts).
PaCoRe hits a premature ceiling by Iteration 3, driven by the verification-centric bottleneck where repeated validation against fixed references yields rapid diminishing returns (detailed in Appendix~\ref{app:wordcloud_analysis}).
In contrast, RSE sustains significant gains through Iteration 6, effectively raising the performance upper bound.
Crucially, although Self-Refine also benefits from increased depth without immediate saturation, it exhibits a severe efficiency deficit, consistently trailing RSE by a significant margin in absolute pass@1.
This indicates that RSE optimizes the search trajectory more effectively: by synthesizing diverse population-level experiences rather than iterating on isolated traces or checking against a static consensus, RSE significantly enhances the marginal utility of each search step.

\paragraph{Scaling Search Width.}
Figure~\ref{fig:analysis_scaling}(b) investigates the impact of scaling the reference number ($N_{ref}$).
The results underscore a critical divergence in scaling behaviors.
Majority Voting exhibits a flat performance trajectory, indicating that simply increasing rollout width yields negligible marginal utility without a mechanism to aggregate inter-path experiences.
While PaCoRe attempts to utilize such experience via concatenation, it faces severe scalability bottlenecks. First, PaCoRe is fundamentally constrained by the model's context window. This limitation imposes a hard ceiling on usable references, compelling us to omit the evaluation at $N_{ref}=256$ as the concatenated input exceeds the 256k context window of \textsc{Qwen3-30B-A3B-Thinking-2507}.
Furthermore, we observe that performance actively regresses beyond a certain threshold. We hypothesize that this stems from attention dispersion arising from the excessively long context \citep{lost-in-the-mid}.
RSE circumvents these limitations by decoupling experience accumulation from context length. It effectively assimilates heterogeneous insights from a growing reference pool without inducing context saturation, thereby maintaining robust performance gains and continuous reasoning refinement as width scales.

\begin{wrapfigure}{r}{0.4\textwidth}
    \vspace{-1.5em}
    \centering
    \includegraphics[width=\linewidth]{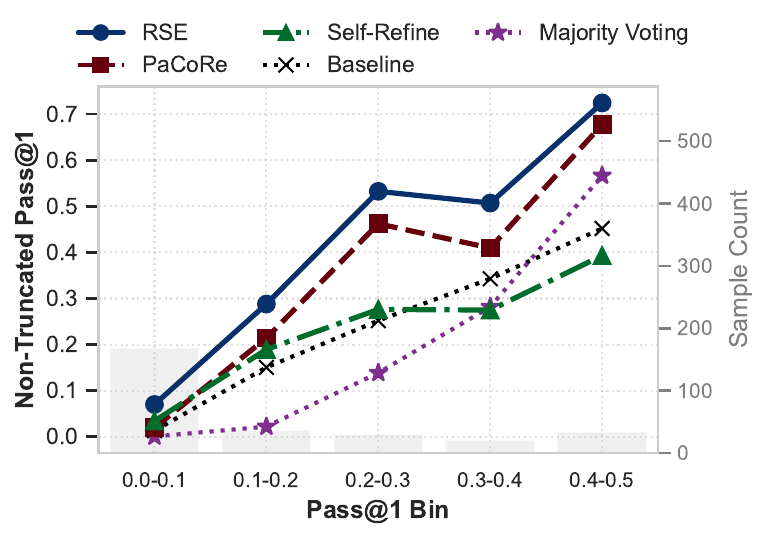}
    \caption{Non-truncated Pass@1 across varying difficulty problems. Samples are stratified by baseline Non-truncated Pass@1, with gray bars indicating the sample count distribution per bin.}
    \label{fig:acc_without_null_text}
    \vspace{-1em}
\end{wrapfigure}

\paragraph{Scaling Efficiency.}
Finally, to rigorously assess the cost-effectiveness of our method, we benchmark performance against the total computational cost (measured in FLOPs; detailed calculation protocols are provided in Appendix~\ref{app:flops_calculation}).
As shown in Figure~\ref{fig:analysis_scaling}(c), RSE establishes a superior Pareto frontier compared to both PaCoRe and Majority Voting. 
RSE demonstrates a steep ascent, delivering the highest accuracy gains per unit of compute budget. This confirms that structured experience recycling represents a significantly more compute-optimal strategy for scaling test-time inference compared to simple aggregation or verification-based concatenation.
Additionally, we report stage-wise overhead and wall-time analysis in Appendix~\ref{sec:overhead_analysis} and Appendix~\ref{sec:wall_clock_analysis}.

\subsubsection{Reasoning Dynamics on Hard Problems.}

Figure~\ref{fig:null_text_rate} reveals that the issue of output truncation is effectively mitigated from the first iteration onwards. 
To decouple genuine reasoning improvements from mere format-level restoration (i.e., the resolution of truncation errors), we employ a specialized metric for this analysis: \textit{Non-Truncated Pass@1}. 
Calculated strictly on completed rollouts, this metric isolates the model's logical derivation capabilities. 
Leveraging this metric, we conduct an aggregated analysis across four benchmarks, specifically filtering for ``Hard Problems'', defined as instances where the baseline (i.e., \textsc{Qwen3-30B-A3B-Thinking-2507} Standard Sampling) \textit{Non-Truncated Pass@1} falls within the $[0, 0.5]$ interval \citep{snell2024scaling}.

Figure~\ref{fig:acc_without_null_text} highlights RSE's superior robustness in hard samples, and this advantage is particularly pronounced in the extremely-hard bracket ($[0, 0.1]$) where valid solutions are scarce. In the extremely-hard bracket, while PaCoRe stagnates due to a behavioral collapse towards passive reference verification, RSE sustains performance gains by preserving independent exploration capacity. 
Analysis of reasoning dynamics confirms that RSE achieves superior efficiency by pruning redundancy, as evidenced by the sharp reduction in computational overhead from $23\text{k}$ to $8.4\text{k}$ tokens in Table~\ref{tab:generation_length}.
Crucially, despite this condensation, lexical analysis in Figure~\ref{fig:wordcloud} demonstrates the persistence of cognitive markers like ``\textit{wait}'', indicating that RSE leverages experience to guide the search while keeping the deductive chain intact. Unlike PaCoRe which abandons deep cognition in favor of checking against a potentially flawed consensus, RSE preserves active deduction, enabling it to discover valid solutions even when initial reference pool lacks high-quality answers.

\subsubsection{Analysis of Experience Composition and Utilization}

\begin{table}[ht]
    \centering
    \small
    \renewcommand{\arraystretch}{0.85}
    \caption{Impact of the de-duplication threshold $\tau$ on experience retention and performance. We compare the number of retained positive experiences ($N_{pos}$) and negative experiences ($N_{neg}$) across iterations on HMMT24 and HMMT25. The best Pass@1 scores for each iteration are bolded. }
    \label{tab:ablation_tau_combined}
    \begin{tabular}{cccccccccccc}
        \toprule
        \multirow{2}{*}{\textbf{$\tau$}} & \multicolumn{3}{c}{\textbf{Iteration 1}} & & \multicolumn{3}{c}{\textbf{Iteration 2}} & & \multicolumn{3}{c}{\textbf{Iteration 3}} \\
        \cmidrule{2-4} \cmidrule{6-8} \cmidrule{10-12}
         & $N_{pos}$ & $N_{neg}$ & Pass@1 & & $N_{pos}$ & $N_{neg}$ & Pass@1 & & $N_{pos}$ & $N_{neg}$ & Pass@1 \\
        \midrule
        \multicolumn{12}{c}{\textbf{HMMT24}} \\
        \midrule
        0.6 & 19.4 & 14.9 & 70.8 & & 19.4 & 15.5 & \textbf{74.1} & & 19.2 & 14.6 & 74.1 \\
        0.7 & 36.6 & 30.1 & 71.1 & & 40.4 & 34.3 & 73.3 & & 41.8 & 34.6 & 74.2 \\
        0.8 & 66.1 & 55.8 & \textbf{71.9} & & 85.2 & 72.5 & 73.7 & & 95.3 & 80.9 & \textbf{74.4} \\
        0.9 & 113.5 & 89.2 & 70.8 & & 177.7 & 123.2 & 72.4 & & 225.8 & 153.2 & 72.9 \\
        1.0 & 161.7 & 94.5 & 70.3 & & 339.8 & 149.9 & 72.4 & & 526.4 & 202.2 & 73.3 \\
        \midrule
        \multicolumn{12}{c}{\textbf{HMMT25}} \\
        \midrule
        0.6 & 18.4 & 15.5 & 80.2 & & 20.7 & 15.4 & 82.9 & & 20.9 & 14.9 & 83.0 \\
        0.7 & 34.6 & 30.6 & 81.5 & & 42.7 & 34.9 & 83.4 & & 45.0 & 36.5 & 84.4 \\
        0.8 & 62.7 & 55.1 & 81.3 & & 89.5 & 74.6 & 82.9 & & 101.4 & 85.1 & 83.9 \\
        0.9 & 107.7 & 80.4 & \textbf{82.5} & & 185.7 & 135.7 & \textbf{84.4} & & 237.7 & 173.1 & \textbf{85.1} \\
        1.0 & 158.4 & 89.8 & \textbf{82.5} & & 330.9 & 165.1 & 84.1 & & 505.1 & 230.0 & 84.4 \\
        \bottomrule
    \end{tabular}
\end{table}

We systematically investigate the impact of context construction across three dimensions: the granularity of experience filtering (deduplication threshold), the mechanism of information retention (diversity vs. frequency), and the ablation of specific experience components (positive vs. negative).

\paragraph{Sensitivity to Deduplication Threshold.}
Table~\ref{tab:ablation_tau_combined} exhibits a non-monotonic performance trajectory peaking around $\tau=0.8$.
This trend highlights a clear trade-off: lower thresholds ($\tau < 0.7$) cause information loss by over-merging distinct paths, while higher thresholds ($\tau \rightarrow 1.0$) introduce redundancy-induced bias.
Thus, a moderate threshold is essential to optimally balance reasoning diversity and context efficiency.

\paragraph{Distinctness vs. Consistency.}
To isolate the impact of information distinctness versus consistency reinforcement, we aligned the experience budget $N$ with the yield of our distinctness-driven selection.
The baseline employs random sampling from the raw pool to explicitly introduce consistency information (i.e., high-frequency experiences).
The sustained superiority of the distinctness-driven approach (Table~\ref{tab:consist_vs_distin}) confirms that information richness facilitates reasoning more effectively than pure consistency, as diverse contexts provide broader logical references to accelerate solution derivation.

\begin{table}[htbp]
    \centering
    \small
    \caption{\textbf{Component-wise Ablation Study.} We evaluate the impact of different experience types on search performance. ``Positive Only'' and ``Negative Only'' denote using exclusively positive experience or negative experience, respectively. Performance is reported as Pass@1 (\%).}
    \label{tab:ablation_components}
    \begin{tabular}{lcccccc}
        \toprule
        \multirow{2}{*}{\textbf{Configuration}} & \multicolumn{3}{c}{\textbf{HMMT-24}} & \multicolumn{3}{c}{\textbf{HMMT-25}} \\
        \cmidrule(lr){2-4} \cmidrule(lr){5-7}
         & \textbf{Iter 1} & \textbf{Iter 2} & \textbf{Iter 3} & \textbf{Iter 1} & \textbf{Iter 2} & \textbf{Iter 3} \\
        \midrule
        Baseline (No Exp.)      & 57.4 & 57.4 & 57.4 & 69.0 & 69.0 & 69.0 \\
        + Positive Exp. Only    & 71.5 & 72.9 & 73.2 & 79.8 & 81.6 & 83.3 \\
        + Negative Exp. Only    & 71.8 & 73.2 & 73.9 & 79.7 & 81.0 & 81.6 \\
        \textbf{+ Both (Full RSE)} & \textbf{71.9} & \textbf{73.7} & \textbf{74.4} & \textbf{81.3} & \textbf{82.9} & \textbf{83.9} \\
        \bottomrule
    \end{tabular}
\end{table}

\paragraph{Ablation of Experience Components.}
Finally, we isolate the contributions of Positive Experiences and Negative Experiences to verify their individual necessity and combined synergy. 
The results, detailed in Table~\ref{tab:ablation_components}, reveal three key insights.
First, both components independently yield substantial improvements over the baseline. 
Second, the two types of experience exhibit strong complementarity. Combining them (Full RSE) consistently outperforms using either in isolation across all iterations and benchmarks. 
This suggests that Positive Experiences and Negative Experiences address distinct reasoning failures: the former guides the model towards verified paths, while the latter explicitly blocks known dead ends, effectively pruning the search space from both directions.

\section{Conclusions}

In this work, we identify and address the systemic inefficiency of current ``memoryless'' test-time search, where valuable intermediate insights are largely discarded after every rollout. To bridge this gap, we introduce Recycling Search Experience (RSE), a self-guided, training-free strategy that turns test-time search from a series of isolated trials into a cumulative process. By actively distilling raw trajectories into structured intermediate conclusions and negative constraints, RSE enables models to prune explored dead ends and accelerate valid derivations without external supervision.
Theoretically, we provide an analysis that formalizes the efficiency gains of RSE, validating its advantage over independent sampling in solving complex reasoning tasks. Empirically, extensive evaluations on challenging mathematical benchmarks demonstrate that RSE consistently establishes a superior Pareto frontier compared to strong baselines. Additional experiments on coding, expert-level scientific QA, and long-horizon planning further show its effectiveness beyond mathematical reasoning. Ultimately, our findings suggest a critical paradigm shift: maximizing the potential of test-time compute requires not merely increasing the volume of rollouts, but optimizing the quality of exploration by transforming disposable trials into cumulative experience.

\paragraph{Limitations and future directions.}
RSE realizes experience recycling through a batched protocol, offering a practical trade-off among sharing stability, synchronization cost, and inference latency. Since batch-level updates cannot influence ongoing trajectories and introduce inter-batch waiting, future work could reduce this latency through online or asynchronous low-bandwidth exchange, enabling finer-grained collaboration among parallel rollouts. Another promising direction is experience-quality control, where conflicting or noisy experience entries are detected, calibrated, or resolved before being reused in later search.


\bibliography{anthology,custom}
\bibliographystyle{acl_natbib}

\clearpage
\appendix

\section{Theory: Recycling Search Experience Improves the Probability of Finding a Correct Solution}
\label{sec:theory}

We provide a theoretical analysis to formalize the efficiency gain of recycling search experience.
The clean analysis abstracts the distillation step as extracting sound reusable experience from rollouts, and studies how accumulating such experience can improve the probability of finding a complete solution.
For analytical clarity, we analyze a sequential version of experience recycling, where the Experience Bank is updated after each rollout.
This abstraction isolates the core recycling mechanism; the batched implementation in Algorithm~\ref{alg:rse} follows the same principle but delays bank updates to round boundaries for practical coordination.
Since self-guided distillation may imperfectly extract experience in practice, Section~\ref{sec:theory_incorrect_experience} further characterizes how the guarantee degrades when harmful incorrect experience enters the bank.

\subsection{Setup: required conclusions and a verified-rollout oracle}

\begin{definition}[Required intermediate conclusions (analysis device)]
Fix a problem instance $x$ and a prompting template. Let $\mathcal{C}=\{c_1,\dots,c_L\}$ be a finite set of
\emph{required intermediate conclusions} such that a fully correct solution can be produced whenever all
elements of $\mathcal{C}$ are available.\footnote{This is an abstract analysis device: we do not assume $\mathcal{C}$
is observable. It captures the idea that solving the problem may require collecting multiple key facts or lemmas.}
A rollout output is considered \emph{correct} if its verified content contains all required conclusions:
\[
\mathrm{Correct}(R)\;:=\;\mathbf{1}\{\mathcal{C}\subseteq R\}.
\]
\end{definition}

For clarity, the closed-form analysis below represents reusable experience through required conclusions in $\mathcal{C}$. Sound negative experiences can be viewed as constraints that prune invalid reasoning regions without ruling out valid solution paths. Harmful negative constraints, such as incorrectly labeling a valid path as a dead end, are handled in Section~\ref{sec:theory_incorrect_experience} as incorrectly extracted experience.

\begin{definition}[Verified rollout oracle]
Let $\mathcal{M}\subseteq 2^{\mathcal{C}}$ be a family of \emph{experience states} that contains
$\varnothing$ and $\mathcal{C}$, and is closed under union.
A single model call, together with a verifier abstraction for the extracted content, is modeled as a stochastic set-valued oracle
\[
F:\ \mathcal{M}\times \Omega \to \mathcal{M},\qquad R = F(S;\omega),
\]
where $S\in\mathcal{M}$ is the current verified \emph{experience} injected into the prompt, and
$R\in\mathcal{M}$ is the post-rollout verified state, possibly including injected items.
This verifier abstraction captures the clean setting where extracted experiences are sound; deviations caused by imperfect extraction are handled in Section~\ref{sec:theory_incorrect_experience}.
\end{definition}

\subsection{Two procedures and the same success notion}

\begin{definition}[Baseline vs.\ Recycling Search Experience (RSE)]
Fix a rollout budget $N\in\mathbb{N}$ and randomness $\omega_1,\dots,\omega_N$.

\noindent\textbf{Baseline (independent-from-scratch sampling).}
For $t=1,\dots,N$, sample
\[
B_t := F(\varnothing;\omega_t).
\]
Define the probability of finding at least one correct solution within budget $N$ as
\[
P_{\mathrm{succ}}^{\mathrm{base}}(N)
\;:=\;
\mathbb{P}\Big(\exists\, t\le N:\ \mathcal{C}\subseteq B_t\Big).
\]

\noindent\textbf{Recycling Search Experience (RSE).}
Initialize the Experience Bank $E_0=\varnothing$. For $t=1,\dots,N$, sample
\[
R_t := F(E_{t-1};\omega_t),\qquad
E_t := E_{t-1}\cup R_t.
\]
Define
\[
P_{\mathrm{succ}}^{\mathrm{rse}}(N)
\;:=\;
\mathbb{P}\Big(\exists\, t\le N:\ \mathcal{C}\subseteq R_t\Big).
\]

\noindent Both procedures declare success under the same predicate:
\emph{``there exists a rollout whose verified output contains all required conclusions''};
the only difference is whether previously verified experience may be injected into later rollouts.
\end{definition}

\subsection{Assumptions}

\paragraph{State convention.}
In the clean analysis, $F(S;\omega)$ denotes the verified state available after running one rollout with injected bank state $S$, rather than the raw text generated by the model. Thus, carried-over bank items are part of the post-rollout verified state by construction.

\begin{assumption}[Sound-state persistence]
\label{ass:persistence}
For all $S\in\mathcal{M}$ and all $\omega\in\Omega$,
\[
S\subseteq F(S;\omega).
\]
This is a bookkeeping convention for the clean process: previously verified bank entries remain available to later rollouts. It does not require the model to explicitly regenerate every injected item in its raw output.
\end{assumption}

\begin{assumption}[No-harm under sound experience injection]
\label{ass:monotone}
For any $S,S'\in\mathcal{M}$ with $S\subseteq S'$ and any $\omega\in\Omega$,
\[
F(S;\omega)\subseteq F(S';\omega).
\]
That is, injecting additional sound experience does not suppress conclusions that would otherwise be verified in the rollout.
\end{assumption}

\paragraph{Remark on the no-harm condition.}
Assumption~\ref{ass:monotone} is a sufficient condition for the clean analysis, not a universal claim about LLM prompting. In practice, additional context or incorrectly extracted experience may change the model's behavior in harmful ways. Section~\ref{sec:theory_incorrect_experience} models such deviations through a bounded harmful-contamination probability.

\subsection{Main result: distribution-free dominance of success probability}

\begin{theorem}[RSE is no worse than independent sampling]
\label{thm:succ-dominance}
Under Assumption~\ref{ass:monotone}, for every budget $N\ge 1$,
\[
P_{\mathrm{succ}}^{\mathrm{rse}}(N)\ \ge\ P_{\mathrm{succ}}^{\mathrm{base}}(N).
\]
Moreover, this dominance holds without any independence assumption on $\omega_1,\dots,\omega_N$.
\end{theorem}

\begin{proof}
Couple the two procedures on the same randomness sequence $(\omega_t)_{t=1}^N$.
For each $t$, since $E_{t-1}\supseteq \varnothing$, Assumption~\ref{ass:monotone} implies the pointwise inclusion
\[
B_t \;=\; F(\varnothing;\omega_t)\ \subseteq\ F(E_{t-1};\omega_t)\;=\;R_t.
\]
Hence whenever $\mathcal{C}\subseteq B_t$ holds for some $t$, we also have $\mathcal{C}\subseteq R_t$ for the same $t$.
Therefore,
\[
\Big\{\exists\, t\le N:\ \mathcal{C}\subseteq B_t\Big\}\ \subseteq\
\Big\{\exists\, t\le N:\ \mathcal{C}\subseteq R_t\Big\}.
\]
Taking probabilities yields $P_{\mathrm{succ}}^{\mathrm{rse}}(N)\ge P_{\mathrm{succ}}^{\mathrm{base}}(N)$.
\end{proof}

\subsection{Independent-coverage model: closed form and sample-efficiency intuition}

This subsection is illustrative. The dominance result in Theorem~\ref{thm:succ-dominance} does not require independence across conclusions or rollouts. The independent-coverage model below is used only to expose the sample-efficiency intuition in closed form: independent sampling must obtain all required conclusions in one rollout, whereas RSE can accumulate them across rollouts.

For this subsection, take $\mathcal{M}=2^{\mathcal{C}}$.

\begin{assumption}[Additive experience model]
\label{ass:additive}
Experience injection guarantees the persistence of previously verified items and does not affect the distribution
of newly discovered verified conclusions:
\[
F(S;\omega)\;=\; S\cup F(\varnothing;\omega),
\qquad \forall S\in\mathcal{M},\ \omega\in\Omega.
\]
\end{assumption}

\begin{corollary}[Closed-form success probabilities under independent coverage]
\label{cor:closed-form}
Assume the from-scratch oracle follows an \emph{independent coverage model}: for each $j\in[L]$ and each rollout $t$, the indicator
$\mathbf{1}\{c_j\in F(\varnothing;\omega_t)\}$ is $\mathrm{Bernoulli}(p_j)$ with $p_j\in(0,1)$, independent across $t$ and $j$.
Under Assumption~\ref{ass:additive},
\begin{equation}
\begin{aligned}
P_{\mathrm{succ}}^{\mathrm{base}}(N)
&= 1-\Big(1-\prod_{j=1}^L p_j\Big)^N,\\
P_{\mathrm{succ}}^{\mathrm{rse}}(N)
&= \prod_{j=1}^L \Big(1-(1-p_j)^N\Big).
\end{aligned}
\end{equation}
In particular,
\[
P_{\mathrm{succ}}^{\mathrm{rse}}(N)\ge P_{\mathrm{succ}}^{\mathrm{base}}(N),
\qquad \forall N\ge 1.
\]

Moreover, let $p_{\min}:=\min_j p_j$ and let the target failure probability be $\delta\in(0,1)$.
It suffices for RSE to take
\[
N\ \ge\ \frac{1}{p_{\min}}\log\frac{L}{\delta}
\quad\Rightarrow\quad
P_{\mathrm{succ}}^{\mathrm{rse}}(N)\ge 1-\delta.
\]
For the baseline, let $q:=\prod_{j=1}^L p_j$. The exact budget needed to achieve
$P_{\mathrm{succ}}^{\mathrm{base}}(N)\ge 1-\delta$ is
\[
N_{\mathrm{base}}(\delta)
=
\left\lceil
\frac{\log(1/\delta)}{-\log(1-q)}
\right\rceil .
\]
In the homogeneous case $p_j=p\in(0,1)$, we have $q=p^L$, and hence, as $L$ grows,
\[
N_{\mathrm{base}}(\delta)
=
(1+o(1))\,p^{-L}\log\frac{1}{\delta},
\]
up to the integer ceiling. In contrast, a sufficient RSE budget is
\[
N_{\mathrm{rse}}^{\mathrm{suff}}(\delta)
=
O\!\left(p^{-1}\log\frac{L}{\delta}\right).
\]
Thus, in this illustrative coverage model, independent sampling requires a rollout that simultaneously covers all $L$ conclusions, while RSE only requires each conclusion to appear at least once across rollouts, yielding an exponential separation in $L$.
\end{corollary}

\begin{proof}
Under the independent coverage model, a single from-scratch rollout contains all conclusions with probability
\[
q:=\prod_{j=1}^L p_j.
\]
Therefore,
\[
P_{\mathrm{succ}}^{\mathrm{base}}(N)=1-(1-q)^N.
\]

Under Assumption~\ref{ass:additive}, the $t$-th RSE rollout satisfies
\[
R_t = F(E_{t-1};\omega_t) = E_{t-1}\cup F(\varnothing;\omega_t).
\]
By induction,
\[
E_t=\bigcup_{s=1}^t F(\varnothing;\omega_s).
\]
Thus RSE succeeds within budget $N$ iff every required conclusion $c_j$ appears at least once in the from-scratch components
$\{F(\varnothing;\omega_t)\}_{t=1}^N$.
For each $j$, the probability that $c_j$ never appears in $N$ trials is $(1-p_j)^N$, so it appears at least once with probability
$1-(1-p_j)^N$. Independence across $j$ yields
\[
P_{\mathrm{succ}}^{\mathrm{rse}}(N)
=
\prod_{j=1}^L \Big(1-(1-p_j)^N\Big).
\]

The inequality
$P_{\mathrm{succ}}^{\mathrm{rse}}(N)\ge P_{\mathrm{succ}}^{\mathrm{base}}(N)$
also follows directly from event inclusion. If independent sampling succeeds, then some rollout contains all conclusions in $\mathcal{C}$; in particular, every conclusion appears at least once among the $N$ from-scratch components. Under the additive experience model, this latter event is exactly the RSE success event.

For the RSE sample-complexity bound,
\begin{equation}
\begin{aligned}
1 - P_{\mathrm{succ}}^{\mathrm{rse}}(N)
&=
\mathbb{P}\Big(\exists j\in[L]:\ c_j\ \text{never appears in $N$ trials}\Big) \\
&\le
\sum_{j=1}^L (1-p_j)^N \\
&\le
L e^{-N p_{\min}}.
\end{aligned}
\end{equation}
Thus choosing
$N\ge \frac{1}{p_{\min}}\log\frac{L}{\delta}$
makes the failure probability at most $\delta$.

For the baseline,
\[
P_{\mathrm{succ}}^{\mathrm{base}}(N)\ge 1-\delta
\quad\Longleftrightarrow\quad
(1-q)^N\le \delta
\quad\Longleftrightarrow\quad
N\ge \frac{\log(1/\delta)}{-\log(1-q)}.
\]
In the homogeneous case $p_j=p$, we have $q=p^L$. Since $q\to 0$ as $L\to\infty$ and
\[
-\log(1-q)=q(1+o(1)),
\]
it follows that, up to the integer ceiling,
\[
N_{\mathrm{base}}(\delta)
=
(1+o(1))\,p^{-L}\log\frac{1}{\delta}.
\]
This gives the stated exponential separation against the sufficient RSE budget.
\end{proof}

\subsection{Effect of Incorrectly Extracted Experience}
\label{sec:theory_incorrect_experience}

The clean analysis above assumes that extracted experiences are sound. We now discuss the effect of incorrectly extracted experience. This type of noise is qualitatively different from missing a useful experience: a false intermediate conclusion or an erroneous negative constraint can bias later rollouts and may reduce the probability of finding a correct solution. Without any control on such harmful errors, no dominance guarantee over independent sampling is possible, since a misleading experience could in principle steer all subsequent rollouts away from correct solutions. We therefore characterize the degradation in terms of a bounded failure probability.

\paragraph{A general failure-event view.}
Let $\widetilde{R}_t\subseteq\mathcal{C}$ denote the set of required conclusions verified in the output produced by the actual RSE pipeline with imperfect experience extraction under the same randomness $\omega_t$ as the baseline rollout
\[
B_t=F(\varnothing;\omega_t).
\]
Define the coupled monotonicity-failure event
\[
G_N
=
\left\{
\exists t\le N:\ B_t \nsubseteq \widetilde{R}_t
\right\}.
\]
This event captures implementation-level deviations that break the clean no-harm coupling, including harmful incorrect experience, context distraction, or anchoring effects.

\begin{proposition}[Approximate dominance under bounded monotonicity failure]
\label{thm:failure_event}
If $\mathbb{P}(G_N)\le \gamma_N$, then
\[
P_{\mathrm{succ}}^{\mathrm{imp}}(N)
\ge
P_{\mathrm{succ}}^{\mathrm{base}}(N)-\gamma_N,
\]
where
\[
P_{\mathrm{succ}}^{\mathrm{imp}}(N)
:=
\mathbb{P}\Big(\exists t\le N:\mathcal{C}\subseteq \widetilde{R}_t\Big)
\]
is the success probability of RSE with imperfect experience extraction.
\end{proposition}

\begin{proof}
Let
\[
\mathcal{B}_N
=
\{\exists t\le N:\mathcal{C}\subseteq B_t\}
\]
be the baseline success event, and let
\[
\widetilde{\mathcal{R}}_N
=
\{\exists t\le N:\mathcal{C}\subseteq \widetilde{R}_t\}
\]
be the success event of imperfect RSE.
On $G_N^c$, we have $B_t\subseteq \widetilde{R}_t$ for every $t\le N$. Hence baseline success implies imperfect RSE success on $G_N^c$:
\[
\mathcal{B}_N\cap G_N^c
\subseteq
\widetilde{\mathcal{R}}_N.
\]
Equivalently,
\[
\mathcal{B}_N
\subseteq
\widetilde{\mathcal{R}}_N \cup G_N.
\]
Taking probabilities gives
\[
P_{\mathrm{succ}}^{\mathrm{base}}(N)
\le
P_{\mathrm{succ}}^{\mathrm{imp}}(N)
+
\mathbb{P}(G_N)
\le
P_{\mathrm{succ}}^{\mathrm{imp}}(N)
+
\gamma_N,
\]
which proves the claim.
\end{proof}

\paragraph{Incorrect experience as a source of failure.}
We now specialize the failure event to incorrect experience extraction. Define the harmful-contamination event
\[
\mathcal{H}_N
=
\{\text{within the first $N$ rollouts, the bank accepts at least one harmful incorrect experience item}\}.
\]
A harmful incorrect item refers to an extracted experience that changes the effective injected state in a way that can invalidate sound experience injection, such as a false intermediate conclusion or an erroneous negative constraint that rules out a valid reasoning path. Incorrect entries that are ignored by the verifier or have no effect on the clean coupling are not counted in the effective state. We make no assumption on the behavior of the system on $\mathcal{H}_N$: in the worst case, harmful contamination may cause all later rollouts to fail. On the complement $\mathcal{H}_N^c$, all effective injected experiences remain sound, so the clean no-harm coupling applies. Thus, in this incorrect-extraction model,
\[
G_N\subseteq \mathcal{H}_N.
\]
Consequently, if $\mathbb{P}(\mathcal{H}_N)\le \epsilon_N$, then Theorem~\ref{thm:failure_event} gives
\[
P_{\mathrm{succ}}^{\mathrm{imp}}(N)
\ge
P_{\mathrm{succ}}^{\mathrm{base}}(N)-\epsilon_N.
\]

\begin{corollary}[Independent coverage with harmful contamination]
\label{cor:incorrect_coverage}
Under the independent coverage model in Corollary~\ref{cor:closed-form}, suppose the only deviation from the sound additive model is the possible acceptance of harmful incorrect experience, captured by $\mathcal{H}_N$, with
$\mathbb{P}(\mathcal{H}_N)\le \epsilon_N$. Equivalently, conditional on $\mathcal{H}_N^c$, the imperfect process coincides with the sound RSE process in Corollary~\ref{cor:closed-form}. Then
\[
P_{\mathrm{succ}}^{\mathrm{imp}}(N)
\ge
\prod_{j=1}^L \left(1-(1-p_j)^N\right)-\epsilon_N
\ge
1-L e^{-N p_{\min}}-\epsilon_N.
\]
Consequently, if
\[
N\ge \frac{1}{p_{\min}}\log\frac{2L}{\delta}
\qquad\text{and}\qquad
\epsilon_N\le \frac{\delta}{2},
\]
then
\[
P_{\mathrm{succ}}^{\mathrm{imp}}(N)\ge 1-\delta.
\]

Moreover, let
\[
q:=\prod_{j=1}^L p_j
\]
and define the clean recycling advantage under the independent coverage model as
\[
\Delta_N
:=
\prod_{j=1}^L \left(1-(1-p_j)^N\right)
-
\left[1-(1-q)^N\right].
\]
Then
\[
P_{\mathrm{succ}}^{\mathrm{imp}}(N)
\ge
P_{\mathrm{succ}}^{\mathrm{base}}(N)+\Delta_N-\epsilon_N.
\]
In particular, imperfect RSE still dominates independent sampling whenever
\[
\epsilon_N\le \Delta_N.
\]
\end{corollary}

\begin{proof}
By Corollary~\ref{cor:closed-form}, in the absence of harmful contamination, the RSE success probability under independent coverage is
\[
P_{\mathrm{succ}}^{\mathrm{rse}}(N)
=
\prod_{j=1}^L \left(1-(1-p_j)^N\right).
\]
By assumption, conditional on $\mathcal{H}_N^c$, the imperfect process coincides with the sound RSE process, while on $\mathcal{H}_N$ we conservatively allow arbitrary failure. Therefore,
\[
P_{\mathrm{succ}}^{\mathrm{imp}}(N)
\ge
P_{\mathrm{succ}}^{\mathrm{rse}}(N)-\mathbb{P}(\mathcal{H}_N)
\ge
\prod_{j=1}^L \left(1-(1-p_j)^N\right)-\epsilon_N.
\]
Using the union bound,
\[
1-
\prod_{j=1}^L \left(1-(1-p_j)^N\right)
\le
\sum_{j=1}^L (1-p_j)^N
\le
L e^{-N p_{\min}},
\]
which yields
\[
P_{\mathrm{succ}}^{\mathrm{imp}}(N)
\ge
1-L e^{-N p_{\min}}-\epsilon_N.
\]
The stated sufficient condition follows by choosing
$N\ge p_{\min}^{-1}\log(2L/\delta)$
and $\epsilon_N\le \delta/2$.

Finally, under the independent coverage model,
\[
P_{\mathrm{succ}}^{\mathrm{base}}(N)
=
1-(1-q)^N.
\]
Thus
\[
P_{\mathrm{succ}}^{\mathrm{imp}}(N)
\ge
P_{\mathrm{succ}}^{\mathrm{rse}}(N)-\epsilon_N
=
P_{\mathrm{succ}}^{\mathrm{base}}(N)+\Delta_N-\epsilon_N.
\]
If $\epsilon_N\le \Delta_N$, then
$P_{\mathrm{succ}}^{\mathrm{imp}}(N)\ge P_{\mathrm{succ}}^{\mathrm{base}}(N)$.
\end{proof}

\paragraph{Interpreting $\epsilon_N$.}
The quantity $\epsilon_N$ captures the cumulative probability that harmful incorrect experience survives the distillation and bank-maintenance pipeline. Let $A_t$ denote the event that the update after rollout $t$ accepts at least one harmful incorrect entry. Let $\mathcal{F}_{t-1}$ denote the history before this update, and let $\mathcal{H}_{t-1}$ be the harmful-contamination event restricted to the first $t-1$ rollouts. If, conditioned on the previous history and no earlier contamination,
\[
\mathbb{P}(A_t\mid \mathcal{F}_{t-1}, \mathcal{H}_{t-1}^c)\le \eta_t,
\]
then a union bound over the first contamination time gives
\[
\epsilon_N
=
\mathbb{P}(\mathcal{H}_N)
\le
\sum_{t=1}^N \eta_t.
\]
In particular, if $\eta_t\le \eta$ for all $t$, then
\[
P_{\mathrm{succ}}^{\mathrm{imp}}(N)
\ge
1-L e^{-N p_{\min}}-N\eta.
\]
This bound is deliberately conservative: it treats any harmful accepted entry as potentially catastrophic, even though in practice later rollouts or deduplication may recover from some erroneous experiences. It therefore highlights the role of experience-quality control: reducing the probability that false conclusions or erroneous negative constraints enter the bank directly preserves the sample-efficiency advantage of experience recycling.

\section{Implementation Details}
\label{app:exp_details}

\subsection{Sampling Configuration}
All experiments are executed in parallel on a cluster with 512 NVIDIA H800 GPUs, where each individual run is allocated to 8 GPUs (except 2 GPUs for Phi-4-Reasoning). To ensure reproducibility and fair comparison, we strictly adhered to the official recommended sampling hyperparameters for each model family:
\begin{itemize}
    \item Qwen3-Thinking-Series: Temperature $T=0.6$, top-$p=0.95$, top-$k=20$, with a maximum generation length of 38k tokens.
    \item Qwen3-Instruct-Series: Temperature $T=0.7$, top-$p=0.8$, top-$k=20$, with a maximum generation length of 38k tokens.
    \item Phi-4-reasoning: Temperature $T=0.8$, top-$p=0.95$, top-$k=50$, with a maximum generation length of 32k tokens.
    \item DeepSeek-V3.2: Temperature $T=1.0$, top-$p=0.95$, with a maximum generation length of 64k tokens. All DeepSeek evaluations were performed via the official API.\footnote{\url{https://platform.deepseek.com/usage}}
\end{itemize}

\noindent \textbf{De-duplication Settings.} To maintain diversity within the retrieved context, we implemented a semantic de-duplication step. We utilized the \texttt{all-MiniLM-L6-v2}\footnote{\url{https://huggingface.co/sentence-transformers/all-MiniLM-L6-v2}} model to encode response candidates and applied a cosine similarity threshold of $0.8$ to filter out redundant samples.

\subsection{Model-Specific Adaptations}
\label{app:model_adaptations}

\paragraph{Evaluation Scope for Deepseek-V3.2.} 
For \textsc{Deepseek-V3.2}, we omit evaluations on HMMT24 and HMMT25 given its reported near-saturation performance (90.2\% pass@1 accuracy on HMMT25) \citep{llm-deepseek}. Additionally, constrained by computational costs, we exclude IMO-AnswerBench and restrict our assessment of \textsc{Deepseek-V3.2} exclusively to the most demanding HLE-Math-text 100-sample subset.

\paragraph{Context Window Adaptation for Phi-4.} 
The limited context window of \textsc{Phi-4-reasoning} (32k) poses a distinct challenge for PaCoRe, which fundamentally relies on aggregating multiple historical reasoning traces. 
As detailed in Table~\ref{tab:token_stats}, our analysis reveals an average reasoning length of $\approx$12.9k tokens), with the 95th percentile (P95) reaching 31.6k. 
Given that the prompt overhead must also be accommodated, these generation lengths effectively saturate the model's entire 32k capacity.
To guarantee sufficient headroom for the generation phase, we reserve a 12k token generation buffer, capping the total input prompt at 20k tokens. 
Leveraging the prior that reasoning paths tend to become more concise and convergent in subsequent iterations, we aggressively truncate individual reference responses to a maximum of 4k tokens. 
This adaptation highlights the critical dependency of context-augmented consistency strategies on the underlying model's context capacity.
Additionally, as Phi-4 implies a fixed system prompt structure, we prepended system instructions to the user prompt.

\begin{table}[h]
    \centering
    \small
    \caption{Token Length Statistics for Phi-4 Reasoning Content. The 95th percentile (P95) length ($\approx$31.6k) nearly exhausts the 32k context window, leaving negligible space for input prompts and necessitating the truncation strategy.}
    \label{tab:token_stats}
    \begin{tabular}{lccc}
        \toprule
        \textbf{Field} & \textbf{Average} & \textbf{Median} & \textbf{P95} \\
        \midrule
        Reasoning Content & 12.9k & 10.7k & 31.6k \\
        Final Text Answer & 1.1k & 1k & 2k \\
        \bottomrule
    \end{tabular}
\end{table}

\subsection{HLE Benchmark Subset Construction}
We utilized the text-only math subset of the Humanity's Last Exam (HLE) benchmark. 
Due to the prohibitive computational cost of performing iterative reasoning evaluations on the full subset (976 samples), we constructed a representative subset using a difficulty-stratified sampling strategy.
First, we assessed the difficulty of all samples using \textsc{Qwen3-30B-A3B-Thinking-2507}. We performed 128 standard sampling rollouts per question and calculated the Pass@1 score.
To focus on problems with significant headroom for improvement, we divided the samples into 10 difficulty bins based on Pass@1, ranging from $0.0$ to $0.5$ with a step size of $0.05$. 
From each bin, we randomly selected 10 samples, resulting in a balanced and tractable subset of 100 samples that preserves the difficulty distribution of the challenging regime.

\subsection{Computational Cost Estimation}
\label{app:flops_calculation}

To explicitly quantify the computational efficiency of different inference strategies, we estimate the Floating Point Operations (FLOPs) following the methodology used in the NVIDIA NeMo framework~\citep{nemo-rl}. 

\paragraph{Token Consumption Dynamics.}
We calculate the cumulative FLOPs by tracking the exact number of tokens processed at each search iteration $t$. Let $n$ be the reference width (number of parallel rollouts). Let $L_{\text{reason}}^{t}$ and $L_{\text{response}}^{t}$ denote the length of the reasoning content and final response generated at iteration $t$, respectively.

\textbf{PaCoRe} 
PaCoRe conditions each generation on the full set of $n$ references from the preceding iteration. For iteration $t+1$:
\begin{align}
    \mathcal{T}_{\text{in}}^{(t+1)} &\approx n \cdot \left( n \cdot L_{\text{response}}^{t} \right) \\
    \mathcal{T}_{\text{out}}^{(t+1)} &= n \cdot \left( L_{\text{reason}}^{t+1} + L_{\text{response}}^{t+1} \right)
\end{align}
Here, the term $n \cdot n$ reflects that each of the $n$ new rollouts essentially processes the concatenation of $n$ historical references.

\textbf{RSE}  
RSE incorporates an intermediate \textit{Experience Distillation} phase while maintaining linear input complexity for generation. 
Let $L_{\text{distill\_prompt}}$ denote the instruction length for experience extraction. 
Crucially, we define $L_{\text{context}}^{(t+1)}$ as the \textit{comprehensive} input length for the reasoning phase, encompassing the system prompt, problem description, and the aggregated experience context constructed from previous iterations.
Let $L_{\text{distill\_out}}^{t}$ represent the total output tokens generated during the distillation process (including both extraction reasoning and the formalized experience).
The total token consumption for iteration $t+1$ is calculated as:

\begin{align}
    \mathcal{T}_{\text{in}}^{(t+1)} &= \underbrace{n \cdot (L_{\text{distill\_prompt}} + L_{\text{reason}}^{t} + L_{\text{response}}^{t})}_{\text{Distillation Input}} \notag \\
    &\quad + \underbrace{n \cdot L_{\text{context}}^{(t+1)}}_{\text{Reasoning Input}} \\
    \mathcal{T}_{\text{out}}^{(t+1)} &= \underbrace{n \cdot L_{\text{distill\_out}}^{t}}_{\text{Distillation Output}} \notag \\
    &\quad + \underbrace{n \cdot (L_{\text{reason}}^{t+1} + L_{\text{response}}^{t+1})}_{\text{Reasoning Output}}
\end{align}

The detailed token statistics used for these calculations are provided in Table~\ref{tab:token_stats_flops}.

\begin{table}[ht]
    \centering
    \small
    \caption{\textbf{Token Statistics for FLOPs Estimation.} We detail the token counts for each phase. For PaCoRe, we separate Thinking ($L_{\text{think}}$) and Response ($L_{\text{resp}}$) lengths, as its input complexity scales quadratically with $L_{\text{resp}}$. For RSE, Iter-0 represents standard sampling (no distillation/context). Note the stark contrast in reasoning length ($L_{\text{think}}$ vs $L_{\text{gen}}$) starting from Iter-1.}
    \label{tab:token_stats_flops}
    \begin{tabular}{lcccccc}
        \toprule
        \multirow{2}{*}{\textbf{Dataset}} & \multirow{2}{*}{\textbf{Iter}} & \multicolumn{2}{c}{\textbf{PaCoRe}} & \multicolumn{3}{c}{\textbf{RSE (Ours)}} \\
        \cmidrule(lr){3-4} \cmidrule(lr){5-7}
         & & $L_{\text{think}}$ & $L_{\text{resp}}$ & $L_{\text{distill\_out}}$ & $L_{\text{context}}$ & $L_{\text{gen}}$ \\
        \midrule
        \multirow{4}{*}{HMMT-24} 
         & 0 & 23,965 & 1,013 & -- & -- & 24,978 \\
         & 1 & 2,157 & 926 & 2,638 & 8,137 & 13,760 \\
         & 2 & 1,729 & 924 & 3,403 & 12,860 & 12,860 \\
         & 3 & 1,522 & 941 & 3,734 & 12,440 & 12,303 \\
        \midrule
        \multirow{4}{*}{HMMT-25} 
         & 0 & 22,264 & 1,038 & -- & -- & 23,302 \\
         & 1 & 1,951 & 960 & 2,701 & 7,973 & 11,228 \\
         & 2 & 2,017 & 994 & 3,611 & 10,313 & 10,313 \\
         & 3 & 2,031 & 1,031 & 4,061 & 11,692 & 9,786 \\
        \bottomrule
    \end{tabular}
\end{table}

\section{Additional Analysis}
\label{app:additional_analyis}

\subsection{Significance Analysis}
\label{app:significance_analysis}

To further assess the reliability of the observed improvements, we conduct a significance analysis on the final iteration results. Specifically, we focus on \textsc{Qwen3-30B-A3B-Thinking-2507} and evaluate the Iteration 3 performance across the considered benchmarks. For each method and benchmark, the reported value is computed over 8 repeated runs and presented as mean performance $\pm$ standard deviation.
In each repeated run, we first average the performance across these four benchmarks and then compute the mean and standard deviation over the 8 run-level averages. This avoids mixing benchmark-level difficulty differences into the run-level variance estimate.

We perform significance analysis between RSE and the strongest iterative baseline for each benchmark. Specifically, the strongest iterative baseline is selected from Self-Ref and PaCoRe according to the average Iteration 3 performance under the same benchmark. The significance marker indicates that the improvement of RSE over the strongest iterative baseline is statistically reliable under the repeated-run evaluation, rather than being attributable to run-level random variation.
As shown in Table~\ref{tab:significance_analysis}, RSE achieves higher final-iteration performance than the strongest iterative baseline across all evaluated benchmarks. The significance results further suggest that the advantage of RSE is consistent across repeated runs, supporting the effectiveness of structured experience recycling over single-trajectory refinement and direct history concatenation.

\begin{table}[ht]
\centering
\small
\caption{
Significance analysis of Iteration 3 results using \textsc{Qwen3-30B-A3B-Thinking-2507}. Values are reported as mean performance $\pm$ standard deviation over 8 repeated runs. The mathematical reasoning result is the macro-average over HMMT24, HMMT25, IMO-AnswerBench, and HLE-Math-text. Significance is assessed between RSE and the strongest iterative baseline selected from Self-Ref and PaCoRe under each benchmark. $^{\dagger}$ indicates that RSE significantly outperforms the strongest iterative baseline with $p < 0.05$ under the repeated-run significance analysis.
}
\label{tab:significance_analysis}
\begin{tabular}{lccc}
\toprule
\textbf{Benchmark} & \textbf{Self-Ref} & \textbf{PaCoRe} & \textbf{RSE} \\
\midrule
Math Benchmarks Avg. & $53.62 \pm 0.24$ & $63.00 \pm 1.48$ & $\mathbf{65.85 \pm 1.07}^{\dagger}$ \\
LiveCodeBench-v6 & $64.13 \pm 0.28$ & $60.26 \pm 0.30$ & $\mathbf{68.55 \pm 0.76}^{\dagger}$ \\
GPQA-Diamond & $70.77 \pm 1.50$ & $69.91 \pm 1.20$ & $\mathbf{72.67 \pm 1.00}^{\dagger}$ \\
TravelPlanner & $28.12 \pm 2.01$ & $0.00 \pm 0.00$ & $\mathbf{34.44 \pm 4.30}^{\dagger}$ \\
\bottomrule
\end{tabular}
\end{table}

\subsection{Comparison with Confidence-Based Strategies}
\label{app:add_baselines}

While RSE optimizes the test-time search process via active experience reuse, methods such as Self-Certainty (Self-Certainty) \citep{baseline-self-c} and Deep Think with Confidence (Deep-Confidence) \citep{baseline-deep-c} primarily focus on post-hoc answer aggregation. We also evaluated these two training-free baselines as a supplementary comparison.

Following our main experimental protocol, we utilized \textsc{Qwen3-30B-A3B-Thinking-2507} on \textsc{HMMT24} and \textsc{HMMT25}. Specifically, we scaled the rollout budgets for Self-Certainty and Deep-Confidence from 64 to 512 to observe their asymptotic behaviors. We report the best-performing tested variant for each benchmark. As shown in Table~\ref{tab:aggregation_comparison}, RSE consistently outperforms these baselines across both datasets. The results suggest that the gains of RSE are not simply due to better answer aggregation or confidence-based compute allocation, but instead come from effectively distilling and reusing search experience across rollouts.

\begin{table}[ht]
\small
\centering
\caption{Comparison with confidence-based strategies.}
\label{tab:aggregation_comparison}
\begin{tabular}{llcccc}
\toprule
\multirow{2}{*}{\textbf{Dataset}} & \multirow{2}{*}{\textbf{Method}} & \multicolumn{4}{c}{\textbf{Rollout Budget}} \\
\cmidrule(lr){3-6}
& & \textbf{@64} & \textbf{@128} & \textbf{@256} & \textbf{@512} \\
\midrule
\multirow{3}{*}{HMMT24} 
& Self-Certainty & 68.3 & 68.1 & 67.8 & 67.3 \\
& Deep-Confidence & 69.3 & 69.3 & 69.6 & 69.7 \\
& RSE (Ours) & \multicolumn{4}{c}{\textbf{74.4}} \\
\midrule
\multirow{3}{*}{HMMT25} 
& Self-Certainty & 76.0 & 76.4 & 76.7 & 76.7 \\
& Deep-Confidence & 75.9 & 77.0 & 78.9 & 80.5 \\
& RSE (Ours) & \multicolumn{4}{c}{\textbf{83.9}} \\
\bottomrule
\end{tabular}
\end{table}

\subsection{Robustness and Configuration Sensitivity Analysis}
\label{sec:robustness_sensitivity}

To comprehensively understand the robustness of RSE and the specific impact of operational configurations, we conduct three targeted evaluations with \textsc{Qwen3-30B-A3B-Thinking-2507}.

\textbf{Robustness against noise interference.} We investigate the model's intrinsic ability to tolerate noise (i.e., misleading experiences) by performing a controlled one-step noise experiment on HMMT24. While maintaining a fixed Experience Bank capacity, we artificially replaced valid experiences with invalid ones at varying proportions (0\%, 20\%, 40\%, 60\%). The method to judge whether an experience is valid follows Appendix~\ref{app:exp_quality_verification}.  The empirical results in Table \ref{tab:noise_robustness} demonstrate that the completely noise-free bank (0\% noise) yields a pass@1 of 73.8, outperforming the automatically distilled Original Bank (71.9). This indicates that further cleaning and filtering of noisy experiences from the initial distillation can effectively improve RSE's performance. 

\begin{table}[ht]
\centering
\caption{Impact of controlled noise injection on HMMT24 pass@1 performance.}
\label{tab:noise_robustness}
\begin{tabular}{lcccccc}
\toprule
\textbf{Setting} & \textbf{Base} & \textbf{Original Bank} & \textbf{0\%} & \textbf{20\%} & \textbf{40\%} & \textbf{60\%} \\
\midrule
HMMT24 & 57.4 & 71.9 & 73.8 & 70.1 & 63.5 & 59.5 \\
\bottomrule
\end{tabular}
\end{table}

\textbf{Sensitivity to experience bank capacity.} To systematically study the performance sensitivity to the Experience Bank size, we conduct a controlled one-step experiment where the bank is constructed from the Iter-0 rollouts. When the number of distilled experiences exceeds a preset limit (60, 90, 120, or 150 experiences), we randomly sample experiences up to that budget for injection. The Full setting utilizes the original distilled bank without any explicit cap. As shown in Table \ref{tab:capacity_sensitivity}, RSE remains competitive even with small experience budgets. Performance scales steadily as the budget constraints are relaxed, with the unrestricted Full setting achieving the best results (71.9 on HMMT24 and 81.3 on HMMT25).

\begin{table}[ht]
\centering
\caption{Sensitivity analysis of Experience Bank capacity (budget limits).}
\label{tab:capacity_sensitivity}
\begin{tabular}{lccccc}
\toprule
\textbf{Bank Size} & \textbf{60} & \textbf{90} & \textbf{120} & \textbf{150} & \textbf{Full} \\
\midrule
HMMT24 & 68.9 & 69.8 & 70.8 & 70.9 & \textbf{71.9} \\
HMMT25 & 79.2 & 79.0 & 80.6 & 81.1 & \textbf{81.3} \\
\bottomrule
\end{tabular}
\end{table}

\textbf{Trade-offs in iteration rollout number.} In our main experiments, the default rollout number for RSE is set to 32 per iteration. To study the sensitivity to this choice, we vary the per-iteration rollout number from 16 to 32 to 64. As shown in Table \ref{tab:rollout_sensitivity}, performance improves as the rollout number increases. This parameter reflects a critical engineering trade-off: if it is too large, more computation is spent on redundant within-batch exploration before the Experience Bank can be updated; if it is too small, achieving sufficient search depth requires more sequential rounds, which inherently increases overall latency. We therefore establish 32 rollouts per iteration as our default configuration.

\begin{table}[ht]
\centering
\caption{Performance comparison across different rollout counts per iteration.}
\label{tab:rollout_sensitivity}
\begin{tabular}{lccc}
\toprule
\textbf{Rollout Num} & \textbf{16} & \textbf{32} & \textbf{64} \\
\midrule
HMMT24 & 71.5 & 74.4 & \textbf{74.6} \\
HMMT25 & 81.7 & 83.9 & \textbf{85.5} \\
\bottomrule
\end{tabular}
\end{table}

\subsection{Quality Verification of Distilled Experiences}
\label{app:exp_quality_verification}

To verify the reliability of our distilled experiences, we employed \textsc{Gemini-3-Pro-Preview} \citep{google2025gemini3pro} as an automated validator. 
The model was prompted to evaluate each experience against the original problem, determining whether Positive Experiences are mathematically valid and whether Negative Experiences describe genuine logical flaws. The validation prompt is provided in Figure\ref{fig:validation_prompt}.

Table~\ref{tab:experience_quality} reports the validation results. 
The distilled experiences achieve an overall accuracy of 81.06\% on HMMT24 and 83.89\% on HMMT25. These high validity scores confirm that our extraction module effectively distills correct and meaningful reasoning signals from the raw rollouts.

\begin{table}[h]
    \centering
    \small
    \caption{Quality Verification of Distilled Experiences. Accuracy of extracted experiences as validated by \textsc{Gemini-3-Pro-Preview}.}
    \label{tab:experience_quality}
    \begin{tabular}{lccc}
        \toprule
        \textbf{Dataset} & \textbf{Positive Exp.} & \textbf{Negative Exp.} & \textbf{Overall} \\
        \midrule
        HMMT24 & 84.48 & 77.61 & 81.06 \\
        HMMT25 & 86.25 & 81.51 & 83.89 \\
        \bottomrule
    \end{tabular}
\end{table}

\subsection{Direct Evidence of Meaningful Experience Reuse}
\label{sec:meaningful_experience_reuse}

\begin{table}[ht]
\centering
\small
\caption{LLM-as-Judge evaluation of experience reuse behavior on HMMT24. Values represent the average number of occurrences per rollout.}
\label{tab:experience_reuse_judge}
\begin{tabular}{lc}
\toprule
\textbf{Category} & \textbf{Avg per Rollout} \\
\midrule
AVOID & 2.18 \\
VERIFY & 1.96 \\
SHORTCUT & 1.49 \\
CORRECT & 0.18 \\
\cmidrule(lr){1-2}
\textbf{Positive-use total} & \textbf{5.81} \\
\midrule
IGNORE & 0.57 \\
MISUSE & 0.28 \\
REFUSE & 0.09 \\
\cmidrule(lr){1-2}
\textbf{Non-positive-use total} & \textbf{0.94} \\
\bottomrule
\end{tabular}
\end{table}

To provide direct evidence of meaningful experience reuse, we conduct an LLM-as-Judge evaluation.

Specifically, we sample 5 rollouts for each of 30 randomly selected questions from HMMT24, yielding a total of 150 reasoning traces. For each trace, a judge model (\textsc{Kimi-K2.5} \citep{kimi-k2.5}) examines the interaction between the provided experiences in the prompt and the model's subsequent chain-of-thought. The judge classifies these interactions into four positive-use categories (AVOID, VERIFY, SHORTCUT, and CORRECT) and three non-positive-use categories (IGNORE, MISUSE, and REFUSE).

As detailed in Table~\ref{tab:experience_reuse_judge}, meaningful positive use substantially outweighs non-positive interactions, achieving a positive-to-non-positive ratio of 6.2:1. This direct evaluation confirms that reused experiences effectively shape and guide the reasoning trajectory in the vast majority of cases.

\subsection{Generalization to Instruction-Tuned Models}
\label{app:instruct_results}

\begin{table}[htbp]
    \centering
    \small
    \renewcommand{\arraystretch}{0.85}
    \caption{Results of instruction-tuned models on HMMT24 and HMMT25.}
    \begin{tabular}{lcccccccc}
        \toprule
        \multirow{2}{*}{\textbf{Method}} & \multicolumn{4}{c}{\textbf{HMMT24}} & \multicolumn{4}{c}{\textbf{HMMT25}} \\
        \cmidrule(lr){2-5} \cmidrule(lr){6-9}
         & It0 & It1 & It2 & It3 & It0 & It1 & It2 & It3 \\
        \midrule
        \multicolumn{9}{c}{\textit{\textbf{Qwen3-30B-Instruct}}} \\
        \midrule
        Base & 26.6 & \na & \na & \na & 34.0 & \na & \na & \na \\
        MV@128 & 48.5 & \na & \na & \na & 45.9 & \na & \na & \na \\
        Self-Ref & \g{26.6} & 39.5 & 41.6 & 43.7 & \g{34.0} & 45.2 & 48.3 & 50.7 \\
        PaCoRe & \g{26.6} & \textbf{52.8} & 53.6 & 53.7 & \g{34.0} & 55.7 & 57.0 & 57.2 \\
        RSE & \g{26.6} & 50.3 & \textbf{55.4} & \textbf{54.4} & \g{34.0} & \textbf{58.5} & \textbf{62.3} & \textbf{64.4} \\
        \midrule
        \multicolumn{9}{c}{\textit{\textbf{Qwen3-4B-Instruct}}} \\
        \midrule
        Base & 21.9 & \na & \na & \na & 26.8 & \na & \na & \na \\
        MV@128 & 31.4 & \na & \na & \na & 32.8 & \na & \na & \na \\
        Self-Ref & \g{21.9} & 26.8 & 28.6 & 29.4 & \g{26.8} & 31.9 & 33.2 & 34.0 \\
        PaCoRe & \g{21.9} & 35.0 & 38.2 & 40.0 & \g{26.8} & 43.2 & 44.4 & 44.8 \\
        RSE & \g{21.9} & \textbf{37.8} & \textbf{40.4} & \textbf{41.5} & \g{26.8} & \textbf{44.1} & \textbf{47.0} & \textbf{49.4} \\
        \bottomrule
    \end{tabular}%
    \label{tab:instruct_results}
\end{table}

\subsection{Generalization to Deepseek-V3.2} 
\label{app:table-ds-v3.2}
We additionally extend our evaluation to \textsc{Deepseek-V3.2} on the HLE-Math-text benchmark, as detailed in Table~\ref{tab:deepseek_results}. Despite the high intrinsic baseline, RSE achieves a substantial gain of 13.5\% at Iteration 2. These results align well with the trends observed in Table~\ref{tab:benchmark_results}, further corroborating the robustness of our approach across varying model scales and capabilities.

\begin{table}[ht]
    \centering
    \small
    \caption{Performance of \textsc{Deepseek-V3.2} on the HLE-Math-text.}
    \label{tab:deepseek_results}
    \begin{tabular}{lccc}
        \toprule
        \textbf{Method} & It0 & It1 & It2 \\
        \midrule
        \multicolumn{4}{c}{\textit{\textbf{Deepseek-V3.2 on HLE-Math-text}}} \\
        \midrule
        Base & 49.3 & \na & \na \\
        MV@128 & 58.6 & \na & \na \\
        PaCoRe & \g{49.3} & 58.1 & 57.9 \\
        RSE & \g{49.3} & \textbf{59.2} & \textbf{62.8} \\
        \bottomrule
    \end{tabular}
\end{table}

\subsection{System Overhead and Latency Analysis}
\label{sec:overhead_analysis}

A clear understanding of the computational and engineering overhead is essential for evaluating the practical applicability of test-time scaling methods. While Appendix B.4 provides a detailed computational cost estimation for RSE based on token consumption across the distillation and reasoning phases, we further present a fine-grained, stage-level breakdown of system latency and FLOPs here. As detailed in Table \ref{tab:latency_breakdown}. The dominant additional cost comes from the extra experience distillation stage after each round, together with the sequential dependency between rounds. By contrast, experience-bank management itself is relatively lightweight: after distillation and semantic deduplication, the bank remains compact, so storing, sampling, and injecting experiences add limited cost compared with the main LLM stages. Our measurements also show that semantic deduplication contributes negligible FLOPs relative to LLM decoding. Overall, the main overhead of RSE comes from the additional LLM pass and serialized scheduling, rather than bank management itself.

\begin{table}[ht]
\centering
\small
\caption{Stage-level breakdown of computational cost (PFLOPs) and latency (seconds) for RSE across iterations.}
\label{tab:latency_breakdown}
\begin{tabular}{llcc}
\toprule
\textbf{Iteration} & \textbf{Stage} & \textbf{Compute (PFLOPs)} & \textbf{Latency (s)} \\
\midrule
\textbf{Iter-0} 
& Rollout & 8.0 & 363.8 \\
\midrule
\multirow{3}{*}{\textbf{Iter-1}} 
& Reflect & 9.9 & 78.3 \\
& Dedup & $9.8 \times 10^{-5}$ & 24.0 \\
& Rollout & 6.6 & 198.2 \\
\midrule
\multirow{3}{*}{\textbf{Iter-2}} 
& Reflect & 5.1 & 96.3 \\
& Dedup & $1.9 \times 10^{-4}$ & 31.0 \\
& Rollout & 8.4 & 190.1 \\
\midrule
\multirow{3}{*}{\textbf{Iter-3}} 
& Reflect & 4.9 & 92.9 \\
& Dedup & $2.8 \times 10^{-4}$ & 34.0 \\
& Rollout & 7.9 & 181.3 \\
\bottomrule
\end{tabular}
\end{table}

\subsection{Wall-Clock Latency Comparison}
\label{sec:wall_clock_analysis}

To evaluate the sequential dependency of our method, we compare the wall-clock latency (in seconds) of RSE against baseline methods in Table~\ref{tab:wall_clock_latency}. RSE incurs a higher elapsed time than purely parallel sampling methods like Majority Voting (MV@128), since each round must wait for experience distillation and bank updates before the next round can begin. This is a real trade-off of the sequential design. 

However, our primary objective is not to achieve the minimum elapsed time under abundant parallel resources, but to improve test-time search by converting previous rollouts into reusable guidance for subsequent sampling, rather than treating each rollout as an independent trial. In practical settings, available parallelism, batch size, and GPU memory are all finite, so efficiency in test-time scaling is often discussed under a strict compute budget (e.g., FLOPs). From this perspective, the additional sequential cost of RSE is precisely what enables later rollouts to become more informed and reduces redundant exploration.

\begin{table}[ht]
\centering
\small
\caption{Comparison of total wall-clock latency (in seconds) across test-time scaling methods. RSE trades increased sequential latency for higher search efficiency under constrained compute budgets.}
\label{tab:wall_clock_latency}
\begin{tabular}{lcccc}
\toprule
\textbf{Dataset} & \textbf{MV@128} & \textbf{PaCoRe} & \textbf{RSE (Ours)} & \textbf{Self-Ref} \\
\midrule
HMMT24 & 365.93 & 732.66 & 1200.99 & 1382.28 \\
HMMT25 & 347.43 & 631.52 & 1052.16 & 1248.26 \\
\bottomrule
\end{tabular}
\end{table}

\subsection{Lexical Analysis of Reasoning Dynamics}
\label{app:wordcloud_analysis}

To investigate the cognitive mechanisms driving the performance divergence observed in Figure~\ref{fig:analysis_scaling}(a), we conducted a lexical analysis of the generated reasoning content. Figure~\ref{fig:wordcloud} presents the word clouds derived from the reasoning traces of PaCoRe and RSE.

\begin{figure}[ht]
    \centering
    \includegraphics[width=0.80\textwidth]{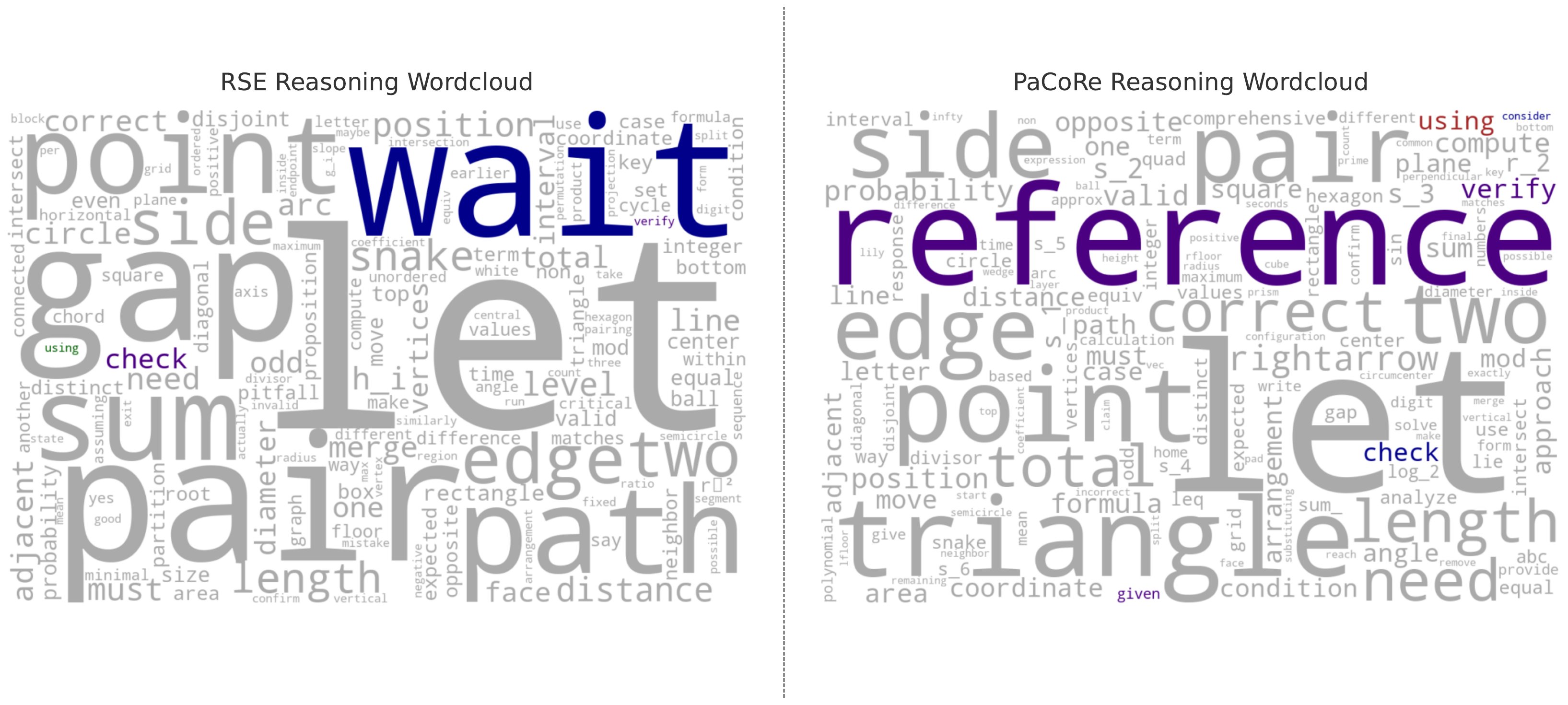}
    \caption{Word Cloud Analysis of Reasoning Content. Left: RSE; Right: PaCoRe.}
    \label{fig:wordcloud}
\end{figure}

\paragraph{The Verification-Centric Bottleneck.}
The PaCoRe word cloud (Right) is dominated by meta-cognitive verification terms such as \texttt{"reference"}, \texttt{"verify"}, and \texttt{"check"}. This lexical distribution reveals a fundamental shift in the model's behavior: instead of engaging in independent problem-solving, the model repurposes its compute budget to validate its generation against the concatenated historical references. 
We term this the \textbf{Verification-Centric Bottleneck}. While this mechanism effectively filters out obvious inconsistencies in early iterations, it creates a closed feedback loop. Once the reference pool stabilizes (even on a suboptimal consensus), the model ceases to generate novel deductive steps, leading to the rapid diminishing returns and premature saturation observed in our scalability experiments.

\paragraph{Preservation of Independent Derivation in RSE.}
In stark contrast, the RSE word cloud (Left) remains dominated by procedural reasoning markers such as \texttt{"wait"} and \texttt{"let"}. 
This lexical distribution confirms that RSE avoids collapsing into a shallow verification of reference content. 
Instead, it preserves the intrinsic reasoning patterns, ensuring that the generation process remains driven by active deduction rather than passive checking.

\subsection{Quantitative Evidence of Mode Collapse.}
The qualitative shift identified in the lexical analysis (Appendix~\ref{app:wordcloud_analysis}) is further corroborated by quantitative metrics on reasoning length (Table~\ref{tab:generation_length}) and answer entropy (Table~\ref{tab:entropy_analysis}).
As shown in Table~\ref{tab:generation_length}, PaCoRe exhibits a precipitous length collapse, with the average reasoning budget shrinking by over 90\% (from 23k to 1.8k tokens) by Iteration 3. 
This drastic reduction aligns with the Verification-Centric Bottleneck: verification is inherently less compute-intensive than generation. Once the model pivots to merely checking against references, it abandons the deep cognitive processes required for independent deduction.
Concurrently, Table~\ref{tab:entropy_analysis} reveals that PaCoRe's answer entropy decays to near-zero (0.0606), quantitatively confirming the occurrence of mode collapse.
In stark contrast, RSE sustains a substantial reasoning length ($\approx$8.4k tokens) and higher entropy (0.3319), demonstrating that it preserves the capacity for active, diverse exploration throughout the iterative process.

\begin{table}[h]
    \centering
    \small
    \caption{Average length of the generated \texttt{reasoning\_content} across different iterations. RSE maintains a substantial reasoning budget, whereas PaCoRe exhibits a sharp length collapse.}
    \label{tab:generation_length}
    \begin{tabular}{lcccc}
        \toprule
        \textbf{Method} & \textbf{Iter 0} & \textbf{Iter 1} & \textbf{Iter 2} & \textbf{Iter 3} \\
        \midrule
        PaCoRe     & \multirow{2}{*}{23,309} & 3,935  & 2,184 & 1,885 \\
        RSE (Ours) &                         & 10,696 & 9,044 & 8,417 \\     
        \bottomrule
    \end{tabular}
\end{table}

\begin{table}[ht]
    \centering
    \small
    \caption{Answer entropy dynamics during search iteration.}
    \begin{tabular}{lcccc}
        \toprule
        \textbf{Method} & \textbf{Iter-0} & \textbf{Iter-1} & \textbf{Iter-2} & \textbf{Iter-3} \\
        \midrule
        RSE & 1.3264 & 0.5047 & 0.3779 & 0.3319 \\
        PaCoRe & 1.3264 & 0.3603 & 0.1359 & 0.0606 \\
        \bottomrule
    \end{tabular}
    \label{tab:entropy_analysis}
\end{table}

\subsection{Impact of Context Composition.}

Table~\ref{tab:consist_vs_distin} provides the detailed numerical breakdown for the ablation study regarding information distinctness versus consistency reinforcement.

\begin{table}[h]
    \centering
    \small
    \caption{We compare context selection strategies under identical shot counts. \textit{Dedup Sampling} uses the set of unique experiences identified by our method ($N$ unique items). \textit{Random Sampling} selects an equal number ($N$) of examples from the raw pool. This setup isolates the effect of information diversity (Distinctness) versus frequency reinforcement (Consistency).}
    \begin{tabular}{lccc}
        \toprule
        \textbf{Dataset} & \textbf{Iter} & \textbf{Random Sampling} & \textbf{Dedup Sampling} \\
         & & \textit{(Consistency-Biased)} & \textit{(Distinctness-Driven)} \\
        \midrule
        \multirow{3}{*}{HMMT24} 
         & It-1 & 69.4 & \textbf{71.9} \\
         & It-2 & 71.0 & \textbf{73.7} \\
         & It-3 & 71.8 & \textbf{74.4} \\
        \midrule
        \multirow{3}{*}{HMMT25} 
         & It-1 & 80.3 & \textbf{81.3} \\
         & It-2 & 81.6 & \textbf{82.9} \\
         & It-3 & 82.4 & \textbf{83.9} \\
        \bottomrule
    \end{tabular}%
    \label{tab:consist_vs_distin}
\end{table}

\subsection{Truncated Completions Analysis}
\label{app:truncation_analysis}

Figure~\ref{fig:null_text_rate} illustrates the truncated completion rates varying across problem difficulty bins (grouped by pass@1 intervals).

\begin{figure}[htbp]
    \centering
    \includegraphics[width=0.60\textwidth]{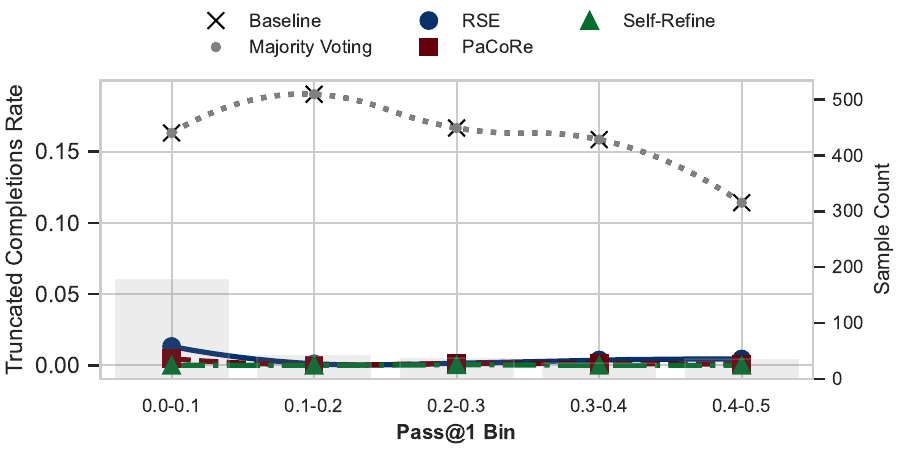}
    \caption{Truncated rate varying difficulty problems.}
    \label{fig:null_text_rate}
\end{figure}

\begin{figure*}[t]
    \centering
    \begin{tcolorbox}[
        colback=gray!5,
        colframe=black,
        boxrule=0.8pt,
        arc=2pt,
        left=10pt, right=10pt, top=10pt, bottom=10pt,
        title=\textbf{Case Analysis}
    ]
        \ttfamily
        \scriptsize
        
        \textbf{Question:} Compute the number of triples $(f, g, h)$ of permutations on $\{1,2,3,4,5\}$ such that
        $f(g(h(x)))=h(g(f(x)))=g(x)$, $g(h(f(x)))=f(h(g(x)))=h(x)$, and $h(f(g(x)))=g(f(h(x)))=f(x)$
        for all $x \in\{1,2,3,4,5\}$.
        
        \vspace{0.3em}
        \textbf{Answer:} 146
        
        \vspace{0.5em}
        \hrule 
        \vspace{0.5em}
        
        \textbf{(1) Intermediate Conclusions}
        \begin{itemize}
            \setlength\itemsep{0em}
            \setlength\parskip{0em}
            \setlength\leftskip{-1em} 
            \item[\textcolor{darkgreen}{\textbullet}] \textcolor{darkgreen}{The number of involutions in $S_5$ is 26.}
            \item[\textcolor{darkgreen}{\textbullet}] \textcolor{darkgreen}{All permutations $f, g, h$ in valid solutions must be involutions, i.e., $f^2 = g^2 = h^2 = \text{id}$.}
            \item[\textcolor{darkgreen}{\textbullet}] \textcolor{darkgreen}{The number of valid triples $(f, g, h)$ equals the number of ordered pairs $(f, g)$ of involutions such that $(fg)^3 = \text{id}$.}
        \end{itemize}

        \vspace{0.3em}
        
        \textbf{(2) Failure Patterns}
        \begin{itemize}
            \setlength\itemsep{0em}
            \setlength\parskip{0em}
            \setlength\leftskip{-1em}
            \item[\textcolor{red}{\textbullet}] \textcolor{red}{Failing to verify that all 6 permutations of the three transpositions on a 3-element set satisfy the equations}
            \item[\textcolor{red}{\textbullet}] \textcolor{red}{Counting only Type 1 solutions (where $f = g = h$) and ignoring Type 2 solutions (where $f, g, h$ are distinct transpositions on a 3-element subset)}
            \item[\textcolor{red}{\textbullet}] \textcolor{red}{Failing to account for the fixed points outside the 3-cycle support}
        \end{itemize}
        
        \vspace{0.5em}
        \hrule 
        \vspace{0.5em}
        
        \textbf{Snapshots of the Model's Reasoning Process}
        \vspace{0.5em}

        \vspace{0.3em}
        \centerline{...}
        \vspace{0.3em}

        \greencirc{1} The number of involutions in $S_5$ is known to be 26 (1 identity, 10 transpositions, 15 double transpositions). Let's confirm... 1 + 10 + 15 = 26. Yes, that's correct.
        
        \vspace{0.3em}
        \centerline{...}
        \vspace{0.3em}

        \greencirc{2} Let's assume $f, g, h$ are involutions ($f^2 = g^2 = h^2 = e$)... The positive experience 'All permutations $f, g, h$ in valid solutions must be involutions' is verified, so we can safely assume that.
        
        \vspace{0.3em}
        \centerline{...}
        \vspace{0.3em}

        \redcirc{3} The critical negative experiences mention that 'Failing to verify that all 6 permutations of the three transpositions on a 3-element set satisfy the equations' is a fatal flaw, but we've verified that they do, so we're safe.
        
        \vspace{0.3em}
        \centerline{...}
        \vspace{0.3em}

        \greencirc{4} The positive experience 'The number of valid triples $(f, g, h)$ equals the number of ordered pairs $(f, g)$ of involutions such that $(fg)^3 = id$.' Let's see... The number of ordered pairs $(f, g)$ ... is equal to the number of Case 1 triples... plus the number of Case 2 triples...
        
        \vspace{0.3em}
        \centerline{...}
        \vspace{0.3em}

        \redcirc{5} The negative experience 'Counting only Type 1 solutions ($f = g = h$) and ignoring Type 2 solutions' would lead to 26, but we've counted both.
        
        \vspace{0.3em}
        \centerline{...}
        \vspace{0.3em}

        \redcirc{6} Another negative experience is 'Failing to account for fixed points outside the 3-element subset,' which we've addressed by multiplying by 2 (the number of involutions on the 2-element complement).
        
    \end{tcolorbox}
    \caption{\textbf{Analysis of Reasoning Components.} The figure illustrates the problem statement, key positive constraints (Green), and critical failure modes (Red). The bottom section displays reasoning slices where the model successfully utilizes the intermediate conclusions (marked \textcolor{darkgreen}{\textcircled{\textbf{1}}}, \textcolor{darkgreen}{\textcircled{\textbf{2}}}, \textcolor{darkgreen}{\textcircled{\textbf{4}}}) and actively avoids the identified failure patterns (marked \textcolor{red}{\textcircled{\textbf{3}}}, \textcolor{red}{\textcircled{\textbf{5}}}, \textcolor{red}{\textcircled{\textbf{6}}}).}
    \label{fig:rse_case}
\end{figure*}

\subsection{RSE Case Analysis}
\label{app:rse_case}

We provide an RSE reasoning case in Figure \ref{fig:rse_case}. The model demonstrates a structured reasoning process by actively integrating Intermediate Conclusions to accelerate search efficiency and utilizing Failure Patterns to ensure solution validity, while exhibiting critical discernment to maintain reasoning robustness against misleading cues.

\paragraph{Accelerating Reasoning Efficiency with Intermediate Conclusions.}
The model leverages Intermediate Conclusions (green) as reasoning shortcuts to significantly enhance search efficiency.
\begin{itemize}
    \item \textbf{Direct Knowledge Retrieval (Step \greencirc{1}):} By directly retrieving the constant $|I(S_5)|=26$, the model substitutes a complex combinatorial sub-task with a single memory lookup. This maximizes efficiency by eliminating an error-prone manual derivation chain, allowing the limited inference budget to be focused on the core logic.
    \item \textbf{Search Space Pruning (Step \greencirc{2}):} By enforcing the involution constraint ($f^2=\text{id}$), the model instantly reduces the candidate space from $|S_5|^3$ to $26^3$. This significantly improves reasoning efficiency by preventing the allocation of inference budget to invalid paths, allowing the model to focus its computational resources strictly on the viable subset.
    \item \textbf{Problem Reformulation (Step \greencirc{4}):} By utilizing the bijective mapping property, the model converts the recursive functional verification into a deterministic algebraic calculation ($(fg)^3=\text{id}$). This linearizes the reasoning path, effectively bypassing the high cognitive load and potential confusion associated with verifying nested recursive depths.
\end{itemize}

\paragraph{Ensuring Validity and Completeness via Adversarial Auditing.}
Failure Patterns (red) serve as an adversarial checklist, forcing the model to cover blind spots that standard chain-of-thought often misses.
\begin{itemize}
    \item \textbf{Overcoming Simplicity Bias (Step \redcirc{5}):} By consulting the pattern on solution types, the model actively expands its search beyond the trivial $f=g=h$ case. This mechanism ensures the completeness of the solution set, preventing the common undercounting error where models settle for the most probable/simplest subset and ignore complex valid configurations (Type 2).
    \item \textbf{Global Consistency Check (Step \redcirc{6}):} Alerted by the risk of ignoring fixed points, the model extends its algebraic focus from the local 3-cycle to the full domain $S_5$. This validates the transition from local derivation to global conclusion, ensuring the structural correctness of the final answer by correctly applying the complement multiplier.
\end{itemize}

\paragraph{Critical Discernment: Robustness Against Misleading Experiences.}
Crucially, the model demonstrates robust reasoning by not blindly following retrieved experiences, but actively verifying their applicability to prevent being misled.
\begin{itemize}
    \item \textbf{Active Discernment (Step \redcirc{3}):} Encountering the failure pattern warning against specific permutation sets, the model does not mechanically discard them. Instead, it exercises critical discernment by performing an independent calculation to verify their validity in the current context. This mechanism directly reinforces reasoning robustness, ensuring that the model remains resilient against potentially misleading retrieval cues and anchors its logical progression strictly on verified mathematical truths.
\end{itemize}

In conclusion, this case exemplifies how RSE optimizes the entire reasoning trajectory. By leveraging retrieved knowledge as both computational accelerators and adversarial auditors, the model achieves a synergy of high efficiency and rigorous completeness. Furthermore, its capacity for critical discernment ensures the reasoning process remains robust, effectively filtering noise while correctly incorporating complex truths.

\section{Prompts}
\label{app:prompts}

In this section, we present the full prompt templates used in our experiments. 

\begin{figure*}[ht]
    \centering
    \begin{tcolorbox}[
        colback=gray!5,
        colframe=black,
        boxrule=0.8pt,
        arc=2pt,
        left=10pt, right=10pt, top=10pt, bottom=10pt,
        title=\textbf{Default System Prompt for Mathematical Reasoning}
    ]
        \ttfamily
        \scriptsize
        Please reason step by step, and put your final answer within \textbackslash boxed\{\}.
    \end{tcolorbox}
    \caption{\textbf{Default system prompt for mathematical reasoning.} We apply this system instruction across all evaluated models to enforce step-by-step reasoning and standardized answer formatting.}
    \label{fig:system_prompt}
\end{figure*}

\begin{figure*}[ht]
    \centering
    \begin{tcolorbox}[
        colback=gray!5,
        colframe=black,
        boxrule=0.8pt,
        arc=2pt,
        left=10pt, right=10pt, top=10pt, bottom=10pt,
        title=\textbf{PaCoRe Input Serialization Template for Mathematical Reasoning}
    ]
        \ttfamily
        \scriptsize
        
        You are given a problem and a list of reference responses. Your job is to analyze these references and provide your own response.
        
        \vspace{1em}
        
        Original Problem: \\
        \textcolor{red}{\{\{ original\_prompt \}\}}
        
        \vspace{1em}
        
        Reference Responses: \\
        \textcolor{red}{\{\% for response in ref\_responses \%\}} \\
        Reference \textcolor{red}{\{\{ loop.index \}\}}: \\
        \textcolor{red}{\{\{ response \}\}} \\
        \textcolor{red}{\{\% endfor \%\}}
        
        \vspace{1em}
        
        Now, based on the original problem and reference responses above, please provide your own comprehensive solution.
        
    \end{tcolorbox}
    \caption{\textbf{PaCoRe input serialization template for mathematical reasoning.} Adopted from the PaCoRe implementation, this template embeds the current problem $x$ (denoted as \texttt{original\_prompt}) and the reference message set $\mathcal{M}$ (denoted as \texttt{ref\_responses}) into the model's context via Jinja2 syntax.}
    \label{fig:pacore_prompt}
\end{figure*}


\begin{figure*}[ht]
    \centering
    \small
    \begin{tcolorbox}[
        colback=gray!5,
        colframe=black,
        boxrule=0.8pt,
        arc=2pt,
        left=8pt, right=8pt, top=8pt, bottom=8pt,
        title=\textbf{Experience Distillation of RSE for Mathematical Reasoning.}
    ]
        \ttfamily
        \scriptsize
        \setlength{\parskip}{0.5em}
        
        "You are a Strategic Reasoning Distiller. Your goal is to construct a "Experience Bank" that will serve as the foundation for the student's next problem-solving iteration by extracting two specific lists:
        \begin{enumerate}[leftmargin=*, nosep]
            \item \textbf{Verified Positive Experiences:} Irrefutable truths and intermediate conclusions derived correctly.
            \item \textbf{Critical Negative Experiences:} Logical fallacies, dangerous operations, and dead ends to avoid.
        \end{enumerate}
        The student will explicitly reference this data: 
        \begin{itemize}[leftmargin=*, nosep]
            \item Utilizing \textbf{Verified Positive Experiences} as established anchors to accelerate valid reasoning.
            \item Consulting \textbf{Critical Negative Experiences} to proactively avoid repeating previously identified errors, logic gaps, or dead ends.
        \end{itemize}

        \textbf{Constraint: strict\_neutrality} \\
        You have \textbf{NO access} to the golden answer. You must \textbf{NOT} make any assumptions about whether the student's final conclusion is correct or incorrect. Treat the student's work as an unverified hypothesis; verify the validity of each step strictly based on logic and mathematical axioms alone.

        \textbf{Task 1: verified\_positive\_experiences (List[str])} \\
        \textbf{Goal:} Extract \textit{only} mathematically sound, reusable facts (Truth Anchors). \\
        \textbf{Strict Inclusion Rules (Filter Aggressively):}
        \begin{enumerate}[leftmargin=*, nosep]
            \item \textbf{Independent Verification:} You must be able to independently verify the statement is true based on standard mathematical axioms or strictly derived from the previous valid steps.
            \item \textbf{Explicit Conditions:} Every positive experience MUST state its necessary conditions.
            \item \textbf{Atomicity:} Break complex thoughts into the smallest reusable units.
            \item \textbf{No "Lucky Guesses":} Do not include conclusions that are "likely true" but lack logical derivation.
            \item \textbf{Self-Contained:} The string must be understandable without reading the original student text.
        \end{enumerate}
        \textbf{Format:} \texttt{"<Complete Statement with Conditions>. (Source: <Derivation/Method>)"}
        
        \textbf{Task 2: critical\_negative\_experiences (List[str])} \\
        \textbf{Goal:} Identify "Negative Constraints" that serve as warning signs for future explorations. \\
        \textbf{Focus on identifying these specific categories:}
        \begin{enumerate}[leftmargin=*, nosep]
            \item \textbf{Dead Ends (Strategy Failures):} Approaches that are technically valid but lead to unmanageable complexity or circular reasoning.
            \item \textbf{Fatal Logic Flaws (Actual Errors):} Fundamental errors such as non-equivalent transformations.
            \item \textbf{Potential Risks (Unsafe Operations):} Correct-looking steps that lack necessary checks (e.g., dividing by zero).
            \item \textbf{Missing Proof Obligations:} Leaps in logic where a case was ignored.
        \end{enumerate}
        \textbf{Format:} \texttt{"<Context/Step> -> <Type> -> <Explanation: Trigger + Invalid Action + Consequence>"}

        \textbf{Output Requirements} \\
        Output \textbf{ONLY} a raw JSON object. No Markdown formatting. Ensure all LaTeX backslashes are escaped properly.

        \textbf{JSON Structure:}
        \begin{verbatim}
{
    "verified_positive_experiences": [
        "<Complete Statement>. (Source: <Derivation>)",
        "..."
    ],
    "critical_negative_experiences": [
        "<Context> -> <Type> -> <Explanation>",
        "..."
    ]
}
        \end{verbatim}

        \textbf{Input Data} \\
        \textbf{Question:} \\
        \textcolor{red}{\{\{ question \}\}} \\
        \textbf{Student's Attempt:} \\
        \textcolor{red}{\{\{ attempt \}\}}

    \end{tcolorbox}
    \caption{RSE experience distillation prompt  for mathematical reasoning.}
    \label{fig:distiller_prompt}
\end{figure*}


\begin{figure*}[ht]
    \centering
    \begin{tcolorbox}[
        colback=gray!5,
        colframe=black,
        boxrule=0.8pt,
        arc=2pt,
        left=6pt, right=6pt, top=6pt, bottom=6pt,
        title=\textbf{RSE Experience-Guided Problem-Solving Prompt for Mathematical Reasoning}
    ]
        \ttfamily
        \scriptsize
        \setlength{\parskip}{0.3em}
        
        You are an advanced mathematical solver augmented with \textbf{Experience Bank}. 
        You are currently in a \textbf{Test-Time Scaling} loop. Previous attempts on this specific problem have been analyzed to extract useful "Positive Experiences" (Intermediate Results) and "Critical Negative Experiences" (Past Errors).
        Your goal is to solve the problem by starting from the definitions. Use previous memories strictly as a \textbf{navigational aid}.

        \textbf{Operational Guidelines:}
        
        \textbf{1. Accelerate via Verified Positive Experiences (The Anchor):} \\
        \textbf{Rule:} Treat Positive Experiences as \textit{structural hypotheses}, not proven facts. \textbf{Priority:} Prioritize positive experiences that offer abstract insights, simplifications, or identities. \textbf{Skepticism:} Be extremely skeptical of raw numerical positive experiences. NEVER use a specific number from the report unless you have independently derived it. \textbf{Action:} If a positive experience offers a shortcut, verify its premise instantly. If valid, use it; if contradictory, discard it immediately.

        \textbf{2. Navigate via Critical Negative Experiences:} \\
        The provided "Critical Negative Experiences" describe specific logical errors or dead-ends. \textbf{You are STRICTLY FORBIDDEN} from repeating them. If you approach a decision point mentioned in a negative experience, you MUST actively choose an alternative strategy.

        \textbf{3. Conflict Resolution \& Robustness:} \\
        \textbf{Scenario:} If you encounter a contradiction (e.g., conflicting values). \textbf{Constraint:} Do NOT simply choose the "easier" value. \textbf{Action:} A contradiction usually means a foundational assumption is incorrect. Backtrack to the very beginning, re-read the problem statement, and challenge your initial setup.

        \textbf{Context from Previous Attempts:} \\
        \textcolor{red}{\{\{ content\_of\_experience\_bank \}\}}

        \textbf{Instruction:} \\
        Reason step by step. Consult the Experience Bank critically: Avoiding the previous error with negative experiences, and use positive experiences only if they accelerate your work. Put your final answer within \textbackslash boxed\{\}.

    \end{tcolorbox}
    \caption{RSE experience-guided solver prompt for mathematical reasoning.}
    \label{fig:solver_prompt}
\end{figure*}


\begin{figure*}[ht]
    \centering
    \begin{tcolorbox}[
        colback=gray!5,
        colframe=black,
        boxrule=0.8pt,
        arc=2pt,
        left=6pt, right=6pt, top=6pt, bottom=6pt,
        title=\textbf{Experience Validation Prompt}
    ]
        \ttfamily
        \scriptsize
        \setlength{\parskip}{0.3em}
        
        \textbf{[System Prompt]} \\
        You are a rigorous mathematical validator. Your task is to evaluate whether each given mathematical statement is logically valid and correct in the context of the provided problem.

        \textbf{Instructions:} \\
        1. Carefully read the original problem. \\
        2. Analyze each statement in the provided list. \\
        3. For each statement, determine if it is mathematically correct and logically sound. \\
        4. Output your decisions as a Python-style boolean list in the following format: \\
        <decision>[True, False, True, ...]</decision>

        \textbf{Important:} \\
        - The list must contain exactly the same number of boolean values as the number of statements provided. \\
        - Use True if the statement is CORRECT, False if it is INCORRECT or FLAWED. \\
        - For positive experiences: Check if the intermediate result or insight is mathematically valid. \\
        - For negative experiences: Check if the described error/negative experience is a genuine logical flaw that should be avoided. \\
        - Be rigorous but fair in your evaluation. \\
        - Output ONLY the <decision>[...]</decision> tag with the boolean list after your analysis.

        \vspace{0.5em}
        \hrule
        \vspace{0.5em}

        \textbf{[User Template]} \\
        \textbf{Original Problem:} \\
        \textcolor{red}{\{\{ problem \}\}}

        \textbf{Statement Type:} \textcolor{red}{\{\{ statement\_type \}\}}

        \textbf{Statements to Validate (\textcolor{red}{\{\{ count \}\}} items):} \\
        \textcolor{red}{\{\{ statements\_list \}\}}

        Please analyze each statement and output your decisions as a boolean list with exactly \textcolor{red}{\{\{ count \}\}} values. \\
        Format: <decision>[True/False, True/False, ...]</decision>

    \end{tcolorbox}
    \caption{Prompt for Experience Validation.}
    \label{fig:validation_prompt}
\end{figure*}

\begin{figure*}[ht]
    \centering
    \small
    \begin{tcolorbox}[
        colback=gray!5,
        colframe=black,
        boxrule=0.8pt,
        arc=2pt,
        left=8pt, right=8pt, top=8pt, bottom=8pt,
        title=\textbf{System Prompt for GPQA.}
    ]
        \ttfamily
        \scriptsize
        \setlength{\parskip}{0.5em}
        
        What is the correct answer to this question: \textcolor{red}{\{question\}} \\
        Choices: \\
        (A) \textcolor{red}{\{A\}} \\
        (B) \textcolor{red}{\{B\}} \\
        (C) \textcolor{red}{\{C\}} \\
        (D) \textcolor{red}{\{D\}} \\
        Please show your choice in the `answer` field with only the choice letter, e.g., \{"answer": "C"\}.

    \end{tcolorbox}
    \caption{System prompt for GPQA.}
    \label{fig:prompt_gpqa}
\end{figure*}

\begin{figure*}[ht]
    \centering
    \small
    \begin{tcolorbox}[
        colback=gray!5,
        colframe=black,
        boxrule=0.8pt,
        arc=2pt,
        left=8pt, right=8pt, top=8pt, bottom=8pt,
        title=\textbf{Experience Distillation Prompt for GPQA (Part 1 of 2)}
    ]
        \ttfamily
        \scriptsize
        \setlength{\parskip}{0.5em}
        
        You are a Strategic Scientific Reasoning Distiller. Your goal is to construct a "Memory Bank" that will serve as the foundation for the student's next problem-solving iteration by extracting two specific lists:
        \begin{enumerate}[leftmargin=*, nosep]
            \item \textbf{Verified Positive Experiences:} Reliable scientific facts, intermediate conclusions, eliminations, and reasoning steps derived correctly from the attempt.
            \item \textbf{Critical Negative Experiences:} Wrong assumptions, invalid eliminations, logical fallacies, unsupported scientific claims, and dead ends to avoid.
        \end{enumerate}
        The student will explicitly reference this data:
        \begin{itemize}[leftmargin=*, nosep]
            \item Utilizing \textbf{Verified Positive Experiences} as established anchors to accelerate valid reasoning
            \item Consulting \textbf{Critical Negative Experiences} to proactively avoid repeating previously identified errors, logic gaps, or dead ends
        \end{itemize}

        \vspace{0.5em}\noindent\textbf{Constraint: strict\_neutrality} \\
        You have \textbf{NO access} to the golden answer. You have \textbf{NO access} to any external feedback such as answer checking, execution, hidden labels, or evaluation outcomes. You must \textbf{NOT} make any assumptions about whether the student's final choice is correct or incorrect. Treat the student's work as an unverified hypothesis; verify the validity of each step strictly based on scientific reasoning, the stated question and options, and the internal consistency of the attempt alone.

        \vspace{0.5em}\noindent\textbf{\#\# Task 1: verified\_positive\_experiences (List[str])}

        \textbf{Goal:} Extract \textit{only} logically sound, reusable facts from the attempt (Truth Anchors).

        \textbf{Strict Inclusion Rules (Filter Aggressively):}
        \begin{enumerate}[leftmargin=*, nosep]
            \item \textbf{Independent Verification:} You must be able to independently verify that the statement is valid based on scientific knowledge, option semantics, definitional facts, internal reasoning consistency, or strictly derived from previous valid steps.
            \item \textbf{Explicit Conditions:} Every positive experience MUST state its necessary conditions when relevant (e.g., "If the molecule is aromatic, then...", "Assuming the process occurs in equilibrium, ...", "Under classical mechanics, ..."). Do not assume hidden conditions.
            \item \textbf{Atomicity:} Break complex thoughts into the smallest reusable units.
            \item \textbf{No "Lucky Guesses":} Do not include conclusions that are merely plausible, pattern-matched, or based on vague intuition without a clear derivation in the text.
            \item \textbf{Self-Contained:} The string must be understandable without reading the original student text. Replace vague references like "this", "it", or "that option" with explicit scientific statements or option labels.
        \end{enumerate}

        \textbf{Content to Extract:}
        \begin{itemize}[leftmargin=*, nosep]
            \item \textbf{Valid Intermediate Scientific Conclusions:} Concrete facts or deductions derived accurately from the attempt (e.g., "Option B is inconsistent with conservation of energy under the stated setup").
            \item \textbf{Correct Eliminations:} Scientifically justified reasons for ruling out one or more answer choices.
            \item \textbf{Correct Equivalences / Reformulations:} Correct reinterpretations of the question, scientific definitions, domain assumptions, or comparisons between options.
            \item \textbf{Constraint-Driven Deductions:} Deductions regarding units, directionality, causality, feasibility, physical plausibility, biochemical compatibility, or domain-specific constraints.
            \item \textbf{Correct Application of Standard Scientific Principles:} Standard laws, definitions, or scientific regularities used where their conditions are visibly met.
        \end{itemize}

        \textbf{Format:}
        \begin{itemize}[leftmargin=*, nosep]
            \item \texttt{"<Complete Statement with Conditions>. (Source: <Derivation/Method>)"}
        \end{itemize}

    \end{tcolorbox}
    \caption{Experience distillation prompt for GPQA (Part 1 of 2).}
    \label{fig:prompt_gpqa_reflection_1}
\end{figure*}

\begin{figure*}[ht]
    \centering
    \small
    \begin{tcolorbox}[
        colback=gray!5,
        colframe=black,
        boxrule=0.8pt,
        arc=2pt,
        left=8pt, right=8pt, top=8pt, bottom=8pt,
        title=\textbf{Experience Distillation Prompt for GPQA (Part 2 of 2)}
    ]
        \ttfamily
        \scriptsize
        \setlength{\parskip}{0.5em}

        \vspace{0.5em}\noindent\textbf{\#\# Task 2: critical\_negative\_experiences (List[str])}

        \textbf{Goal:} Identify "Negative Constraints" that serve as warning signs for future explorations.

        \textbf{Focus on identifying these specific categories:}
        \begin{enumerate}[leftmargin=*, nosep]
            \item \textbf{Dead Ends (Strategy Failures):} Approaches that are technically possible but unproductive, overly speculative, or based on irrelevant detail rather than discriminating between answer choices.
            \item \textbf{Fatal Logic Flaws (Actual Errors):} Fundamental errors that invalidate the attempt, such as misreading the question, confusing scientific concepts, invalid option elimination, or contradiction with known principles.
            \item \textbf{Potential Risks (Unsafe Operations):} Correct-looking steps that lack necessary justification, such as eliminating an option without checking all conditions, assuming a mechanism/property without evidence, or treating a heuristic as proof.
            \item \textbf{Missing Proof Obligations:} Leaps in logic where an important comparison between choices, scientific precondition, unit check, or causal justification was ignored.
        \end{enumerate}

        \textbf{Format:}
        \begin{itemize}[leftmargin=*, nosep]
            \item \texttt{"<Context/Step> -> <Type: Dead End / Fatal Flaw / Potential Risk> -> <Explanation: Trigger + Invalid Action + Consequence>"}
        \end{itemize}

        \textbf{Explanation Requirements (The "WHY"):}
        \begin{itemize}[leftmargin=*, nosep]
            \item \textbf{Trigger:} What specific concept, option comparison, scientific assumption, or wording in the attempt caused the issue?
            \item \textbf{Invalid Action:} What did the student fail to check, or do incorrectly?
            \item \textbf{Consequence:} What is the reasoning result? (e.g., "Eliminates the correct option without justification," "Keeps an incompatible option alive," "Confuses correlation with mechanism," "Uses a scientific principle outside its valid regime").
        \end{itemize}

        \textbf{Example:}
        \begin{itemize}[leftmargin=*, nosep]
            \item \texttt{"Ruling out option B because it 'seems too strong' -> Potential Risk -> Trigger: qualitative impression only; Invalid Action: eliminated an option without checking whether the underlying scientific claim is actually false under the stated conditions; Consequence: valid options may be discarded for non-scientific reasons."}
        \end{itemize}

        \vspace{0.5em}\noindent\textbf{\#\# Output Requirements}
        \begin{itemize}[leftmargin=*, nosep]
            \item \textbf{Output ONLY a raw JSON object.}
            \item No Markdown formatting (no \texttt{```json ... ```}), no explanations, no chat.
        \end{itemize}

        \textbf{JSON Structure:}
\begin{verbatim}
{
    "verified_positive_experiences": [
        "<Complete Statement with Conditions>. (Source: <Derivation/Method>)",
        "..."
    ],
    "critical_negative_experiences": [
        "<Context/Step> -> <Type: Dead End/Fatal Flaw/Potential Risk> -> <Explanation>",
        "..."
    ]
}
\end{verbatim}

        \vspace{0.5em}\noindent\textbf{\#\# Input Data}

        \textbf{Question:} \\
        \textcolor{red}{\{\{ question \}\}}

        \textbf{Student's Attempt:} \\
        \textcolor{red}{\{\{ attempt \}\}}

    \end{tcolorbox}
    \caption{Experience distillation prompt for GPQA (Part 2 of 2).}
    \label{fig:prompt_gpqa_reflection_2}
\end{figure*}

\begin{figure*}[ht]
    \centering
    \small
    \begin{tcolorbox}[
        colback=gray!5,
        colframe=black,
        boxrule=0.8pt,
        arc=2pt,
        left=8pt, right=8pt, top=8pt, bottom=8pt,
        title=\textbf{Experience-guided problem-solving prompt for GPQA.}
    ]
        \ttfamily
        \scriptsize
        \setlength{\parskip}{0.5em}
        
        You are an advanced scientific reasoning solver. \\
        You have access to a "Reference Report" from previous attempts, containing "Positive Experiences" (Intermediate Results) and "Critical Negative Experiences" (Past Errors).

        \vspace{0.5em}\noindent\textbf{Core Directive: First-Principles Reasoning with Strategic Consultation} \\
        Your primary goal is to determine the correct answer starting from the scientific content of the question, the semantic differences between the answer choices, and the underlying domain principles. Use the provided context strictly as a \textbf{navigational aid}, not as a definitive source of truth.

        \vspace{0.5em}\noindent\textbf{Operational Guidelines:}

        \textbf{1. Positive Experience Handling (Structure > Surface Form):}
        \begin{itemize}[leftmargin=*, nosep]
            \item \textbf{Rule:} Treat Positive Experiences as \textit{structural hypotheses}, not proven facts.
            \item \textbf{Priority:} Prioritize positive experiences that offer \textbf{scientific insights}, \textbf{correct eliminations}, \textbf{causal or mechanistic distinctions}, \textbf{constraint-driven simplifications}, or \textbf{correct reformulations of the question}.
            \item \textbf{Skepticism:} Be extremely skeptical of \textbf{surface-pattern matches}, \textbf{example-specific conclusions}, \textbf{unsupported scientific claims}, or \textbf{assertions that a particular option must be correct without justification}. NEVER trust a proposed final choice unless you have independently justified why it is supported by the question and why competing options fail.
            \item \textbf{Action:} If a positive experience offers a shortcut, verify its \textit{premise} immediately. If the premise holds and aligns with your reasoning, use it to accelerate. If it contradicts the wording of the question, the scientific constraints, or the option semantics, \textbf{discard it immediately}.
        \end{itemize}

        \textbf{2. Negative Experience Discrimination (Mechanism > Appearance):}
        \begin{itemize}[leftmargin=*, nosep]
            \item \textbf{Insight:} Negative Experiences often describe a \textit{misread condition}, a \textit{false elimination}, a \textit{confused scientific concept}, or an \textit{unsupported leap from intuition to conclusion}. Do not confuse a "nontrivial but valid scientific inference" with a "logic error".
            \item \textbf{Action:} When your reasoning resembles a Negative Experience:
            \begin{itemize}[leftmargin=*, nosep]
                \item \textit{Check:} Are you actually committing the specific scientific or logical error described (e.g., confusing correlation with mechanism, applying a principle outside its regime, eliminating an option without checking all relevant conditions)?
                \item \textit{Or:} Are you performing a valid step that merely \textit{looks} similar to the negative experience?
            \end{itemize}
            \item \textbf{Protocol:} If it is a genuine flaw, \textbf{ABORT} the branch. If it is a valid operation, \textbf{PROCEED} but explicitly verify why your reasoning is sound.
        \end{itemize}

        \textbf{3. Conflict Resolution \& Robustness:}
        \begin{itemize}[leftmargin=*, nosep]
            \item \textbf{Scenario:} You encounter a contradiction (e.g., two incompatible interpretations of the question, two competing principles, or two answer choices that seem plausible for different reasons).
            \item \textbf{Constraint:} Do NOT simply choose the more familiar or more intuitive option.
            \item \textbf{Action:} A contradiction usually means a \textbf{foundational assumption} (e.g., what is being asked, what condition is operative, or which principle actually applies) is incorrect. \textbf{Backtrack to the question wording}, compare the options carefully, and challenge your initial setup.
        \end{itemize}

        \textbf{4. Discrimination over Enumeration:}
        \begin{itemize}[leftmargin=*, nosep]
            \item \textbf{Guideline:} Before diving into vague speculation or excessive detail, pause and ask: "What is the key scientific distinction that separates the correct option from the others?"
            \item \textbf{Goal:} Use the Reference Report to find these \textit{structural discriminators} rather than complicating the problem with ad hoc or weakly supported arguments.
        \end{itemize}

        \vspace{0.5em}\noindent\textbf{Context from Previous Attempts:} \\
        \textcolor{red}{\{\{ experience\_context \}\}}

        \vspace{0.5em}\noindent\textbf{Instruction:} \\
        Reason step by step. Start by identifying the core scientific issue, the critical constraints or definitions, and the key distinctions among the answer choices. Consult the Reference Report critically: verify negative experiences before pruning, and use positive experiences only if they safely accelerate your work. Then produce your final answer as a JSON object with the format \{"answer": "A"\} where the value is a single letter from A, B, C, or D.

    \end{tcolorbox}
    \caption{Experience-guided problem-solving prompt for GPQA.}
    \label{fig:prompt_gpqa_with_reflection}
\end{figure*}

\begin{figure*}[ht]
    \centering
    \begin{tcolorbox}[
        colback=gray!5,
        colframe=black,
        boxrule=0.8pt,
        arc=2pt,
        left=10pt, right=10pt, top=10pt, bottom=10pt,
        title=\textbf{PaCoRe Input Serialization Template for GPQA}
    ]
        \ttfamily
        \scriptsize
        
        You are given a GPQA multiple-choice scientific reasoning problem and a list of reference responses. Your job is to analyze these references and provide your own response.
        
        \vspace{1em}
        
        Original Problem: \\
        \textcolor{red}{\{\{ original\_prompt \}\}}
        
        \vspace{1em}
        
        Reference Responses: \\
        \textcolor{red}{\{\% for response in ref\_responses \%\}} \\
        Reference \textcolor{red}{\{\{ loop.index \}\}}: \\
        \textcolor{red}{\{\{ response \}\}} \\
        \textcolor{red}{\{\% endfor \%\}}
        
        \vspace{1em}
        
        Now, based on the original problem and reference responses above, reason independently and provide your own final answer. \\
        Your final answer must be a JSON object with the format \{"answer": "A"\} where the value is exactly one of A, B, C, or D.
        
    \end{tcolorbox}
    \caption{PaCoRe input serialization template for GPQA.}
    \label{fig:pacore_prompt_gpqa}
\end{figure*}

\begin{figure*}[ht]
    \centering
    \begin{tcolorbox}[
        colback=gray!5,
        colframe=black,
        boxrule=0.8pt,
        arc=2pt,
        left=10pt, right=10pt, top=10pt, bottom=10pt,
        title=\textbf{System Prompt and Input Template for LiveCodeBench}
    ]
        \ttfamily
        \scriptsize
        \setlength{\parskip}{0.5em}
        
        \textbf{[System Message]} \\
        You are an expert Python programmer. You will be given a question (problem specification) and will generate a correct Python program that matches the specification and passes all tests.
        
        \vspace{1em}
        \textbf{[User Message Template]} \\
        \#\#\# Question: \\
        \textcolor{red}{\{\{ question\_content \}\}}
        
        \vspace{0.5em}
        \textit{--- If starter code is provided: ---} \\
        \#\#\# Format: You will use the following starter code to write the solution to the problem and enclose your code within delimiters. \\
        \pyfence\\
        \textcolor{red}{\{\{ starter\_code \}\}}\\
        \codefence
        
        \vspace{0.5em}
        \textit{--- If no starter code is provided: ---} \\
        \#\#\# Format: Read the inputs from stdin solve the problem and write the answer to stdout (do not directly test on the sample inputs). Enclose your code within delimiters as follows. Ensure that when the python program runs, it reads the inputs, runs the algorithm and writes output to STDOUT. \\
        \pyfence\\
        \# YOUR CODE HERE\\
        \codefence
        
        \vspace{0.5em}
        \#\#\# Answer: (use the provided format with backticks)
        
    \end{tcolorbox}
    \caption{\textbf{Prompt template for LiveCodeBench.} The template consists of a static system message establishing the expert programmer persona, followed by a dynamically formatted user message that adapts its instructions based on the availability of starter code.}
    \label{fig:prompt_livecodebench}
\end{figure*}

\begin{figure*}[ht]
    \centering
    \small
    \begin{tcolorbox}[
        colback=gray!5,
        colframe=black,
        boxrule=0.8pt,
        arc=2pt,
        left=8pt, right=8pt, top=8pt, bottom=8pt,
        title=\textbf{Experience Distillation Prompt for LiveCodeBench (Part 1 of 2)}
    ]
        \ttfamily
        \scriptsize
        \setlength{\parskip}{0.5em}
        
        You are a Strategic Code Reasoning Distiller. Your goal is to construct a "Memory Bank" that will serve as the foundation for the student's next problem-solving iteration by extracting two specific lists:
        \begin{enumerate}[leftmargin=*, nosep]
            \item \textbf{Verified Positive Experiences:} Reliable facts and intermediate conclusions derived correctly from the algorithmic reasoning or implementation attempt.
            \item \textbf{Critical Negative Experiences:} Wrong assumptions, invalid transitions, unsafe operations, implementation bugs, and dead ends to avoid.
        \end{enumerate}
        The student will explicitly reference this data:
        \begin{itemize}[leftmargin=*, nosep]
            \item Utilizing \textbf{Verified Positive Experiences} as established anchors to accelerate valid reasoning and implementation
            \item Consulting \textbf{Critical Negative Experiences} to proactively avoid repeating previously identified errors, code-level bugs, logic gaps, or dead ends.
        \end{itemize}

        \vspace{0.5em}\noindent\textbf{Constraint: strict\_neutrality} \\
        You have \textbf{NO access} to the golden answer. You have \textbf{NO access} to any external feedback such as code execution, test-case results, hidden tests, or judge outcomes. You must \textbf{NOT} make any assumptions about whether the student's final code or conclusion is correct or incorrect based on imagined runtime behavior. Treat the student's work as an unverified hypothesis; verify the validity of each step strictly based on algorithmic logic, code semantics, the stated problem constraints, and the internal consistency of the attempt alone.

        \vspace{0.5em}\noindent\textbf{\#\# Task 1: verified\_positive\_experiences (List[str])}

        \textbf{Goal:} Extract \textit{only} logically sound, reusable facts from the attempt (Truth Anchors).

        \textbf{Strict Inclusion Rules (Filter Aggressively):}
        \begin{enumerate}[leftmargin=*, nosep]
            \item \textbf{Independent Verification:} You must be able to independently verify that the statement is valid based on the algorithm, data structure properties, control flow, complexity reasoning, boundary conditions, or strictly derived from previous valid steps.
            \item \textbf{Explicit Conditions:} Every positive experience MUST state its necessary conditions (e.g., "If the array is sorted, then...", "For \texttt{0 <= i < n}, ...", "Assuming all edge weights are non-negative, ..."). Do not assume global constraints apply unless stated.
            \item \textbf{Atomicity:} Break complex thoughts into the smallest reusable units.
            \item \textbf{No "Lucky Guesses":} Do not include conclusions that are merely plausible, based on a few examples, or asserted without a clear derivation in the text.
            \item \textbf{Self-Contained:} The string must be understandable without reading the original student text. Replace vague references like "this", "it", or "the loop" with explicit variables, arrays, states, transitions, or operations.
        \end{enumerate}

        \textbf{Content to Extract:}
        \begin{itemize}[leftmargin=*, nosep]
            \item \textbf{Valid Intermediate Algorithmic Conclusions:} Concrete properties derived accurately from previous steps and useful for implementation (e.g., "After sorting nums in nondecreasing order, equal values become adjacent").
            \item \textbf{Correct Equivalences / Reformulations:} Correctly transformed problem statements, invariants, recurrence relations, greedy conditions, or state definitions.
            \item \textbf{Constraint-Driven Deductions:} Deductions regarding valid ranges, monotonicity, uniqueness, feasibility, pruning conditions, or boundary behavior.
            \item \textbf{Correct Application of Standard Techniques:} Standard algorithmic facts or implementation principles used where all conditions are visibly met.
            \item \textbf{Implementation-Relevant Invariants:} Facts that help preserve correctness during coding, iteration, recursion, or state updates.
        \end{itemize}

        \textbf{Format:}
        \begin{itemize}[leftmargin=*, nosep]
            \item \texttt{"<Complete Statement with Conditions>. (Source: <Derivation/Method>)"}
        \end{itemize}

    \end{tcolorbox}
    \caption{Experience distillation prompt for livecodebench (Part 1 of 2).}
    \label{fig:prompt_code_reflection_1}
\end{figure*}

\begin{figure*}[ht]
    \centering
    \small
    \begin{tcolorbox}[
        colback=gray!5,
        colframe=black,
        boxrule=0.8pt,
        arc=2pt,
        left=8pt, right=8pt, top=8pt, bottom=8pt,
        title=\textbf{Experience Distillation Prompt for LiveCodeBench (Part 2 of 2)}
    ]
        \ttfamily
        \scriptsize
        \setlength{\parskip}{0.5em}

        \vspace{0.5em}\noindent\textbf{\#\# Task 2: critical\_negative\_experiences (List[str])}

        \textbf{Goal:} Identify "Negative Constraints" that serve as warning signs for future explorations.

        \textbf{Focus on identifying these specific categories:}
        \begin{enumerate}[leftmargin=*, nosep]
            \item \textbf{Dead Ends (Strategy Failures):} Approaches that are technically possible but lead to unmanageable complexity, excessive casework, fragile implementation, or a solution path unlikely to satisfy the stated constraints.
            \item \textbf{Fatal Logic Flaws (Actual Errors):} Fundamental errors that invalidate the attempt, such as incorrect state definitions, wrong transitions, invalid greedy assumptions, broken invariants, or reasoning that does not preserve correctness.
            \item \textbf{Potential Risks (Unsafe Operations):} Correct-looking steps that lack necessary checks, such as out-of-bounds access, division by zero, incorrect initialization, integer overflow risk, recursion depth issues, misuse of library behavior, or applying an algorithm without verifying its preconditions.
            \item \textbf{Missing Proof Obligations:} Leaps in logic where an important case, invariant, edge case, complexity condition, or correctness condition was ignored.
        \end{enumerate}

        \textbf{Format:}
        \begin{itemize}[leftmargin=*, nosep]
            \item \texttt{"<Context/Step> -> <Type: Dead End / Fatal Flaw / Potential Risk> -> <Explanation: Trigger + Invalid Action + Consequence>"}
        \end{itemize}

        \textbf{Explanation Requirements (The "WHY"):}
        \begin{itemize}[leftmargin=*, nosep]
            \item \textbf{Trigger:} What specific code structure, assumption, constraint pattern, or input pattern caused the issue?
            \item \textbf{Invalid Action:} What did the student fail to check, or do incorrectly?
            \item \textbf{Consequence:} What is the algorithmic or implementation result? (e.g., "Out-of-bounds access," "State transition becomes invalid").
        \end{itemize}

        \textbf{Example:}
        \begin{itemize}[leftmargin=*, nosep]
            \item \texttt{"Using binary search on the answer -> Potential Risk -> Trigger: binary search requires a monotonic feasibility predicate; Invalid Action: Failed to verify monotonicity of the check function; Consequence: Search logic is not justified and may return an incorrect result."}
        \end{itemize}

        \vspace{0.5em}\noindent\textbf{\#\# Output Requirements}
        \begin{itemize}[leftmargin=*, nosep]
            \item \textbf{Output ONLY a raw JSON object.}
            \item No Markdown formatting (no \texttt{```json ... ```}), no explanations, no chat.
        \end{itemize}

        \textbf{JSON Structure:}
\begin{verbatim}
{
    "verified_positive_experiences": [
        "<Complete Statement with Conditions>. (Source: <Derivation/Method>)",
        "..."
    ],
    "critical_negative_experiences": [
        "<Context/Step> -> <Type: Dead End/Fatal Flaw/Potential Risk> -> <Explanation>",
        "..."
    ]
}
\end{verbatim}

        \vspace{0.5em}\noindent\textbf{\#\# Input Data}

        \textbf{Question:} \\
        \textcolor{red}{\{\{ question \}\}}

        \textbf{Student's Attempt:} \\
        \textcolor{red}{\{\{ attempt \}\}}

    \end{tcolorbox}
    \caption{Experience distillation prompt for livecodebench (Part 2 of 2).}
    \label{fig:prompt_code_reflection_2}
\end{figure*}

\begin{figure*}[ht]
    \centering
    \small
    \begin{tcolorbox}[
        colback=gray!5,
        colframe=black,
        boxrule=0.8pt,
        arc=2pt,
        left=8pt, right=8pt, top=8pt, bottom=8pt,
        title=\textbf{Experience-Guided Problem-Solving Prompt for LiveCodeBench (Part 1 of 2)}
    ]
        \ttfamily
        \scriptsize
        \setlength{\parskip}{0.5em}
        
        \textbf{[System Message]} \\
        You are an advanced competitive programming and code generation solver. \\
        You have access to a "Reference Report" from previous attempts, containing "Positive Experiences" (Intermediate Results) and "Critical Negative Experiences" (Past Errors).

        \vspace{0.5em}\noindent\textbf{Core Directive: First-Principles Reasoning with Strategic Consultation} \\
        Your primary goal is to derive a correct solution starting from the fundamental constraints, input/output requirements, and algorithmic structure of the problem. Use the provided context strictly as a \textbf{navigational aid}, not as a definitive source of truth.

        \vspace{0.5em}\noindent\textbf{Operational Guidelines:}

        \textbf{1. Positive Experience Handling (Structure > Surface Form):}
        \begin{itemize}[leftmargin=*, nosep]
            \item \textbf{Rule:} Treat Positive Experiences as \textit{structural hypotheses}, not proven facts.
            \item \textbf{Priority:} Prioritize positive experiences that offer \textbf{algorithmic insights}, \textbf{state definitions}, \textbf{invariants}, \textbf{constraint-driven simplifications}, or \textbf{correct problem reformulations}.
            \item \textbf{Skepticism:} Be extremely skeptical of \textbf{raw example-specific conclusions}, \textbf{unverified implementation claims}, or \textbf{assertions that a particular code pattern is correct}. NEVER trust a proposed final approach unless you have independently justified why it satisfies the problem constraints.
            \item \textbf{Action:} If a positive experience offers a shortcut, verify its \textit{premise} immediately. If the premise holds and aligns with your reasoning, use it to accelerate. If it contradicts your understanding of the problem, constraints, or algorithmic logic, \textbf{discard it immediately}.
        \end{itemize}

        \textbf{2. Negative Experience Discrimination (Mechanism > Appearance):}
        \begin{itemize}[leftmargin=*, nosep]
            \item \textbf{Insight:} Negative Experiences often describe a \textit{misread constraint}, a \textit{broken invariant}, an \textit{invalid transition}, or an \textit{unsafe implementation assumption}. Do not confuse a "nontrivial but valid algorithmic step" with a "logic error".
            \item \textbf{Action:} When your reasoning resembles a Negative Experience:
            \begin{itemize}[leftmargin=*, nosep]
                \item \textit{Check:} Are you actually committing the specific logical or implementation error described (e.g., invalid greedy reasoning, incorrect indexing, missing edge-case handling, misuse of complexity assumptions)?
                \item \textit{Or:} Are you performing a valid step that merely \textit{looks} similar to the negative experience?
            \end{itemize}
            \item \textbf{Protocol:} If it is a genuine flaw, \textbf{ABORT} the branch. If it is a valid operation, \textbf{PROCEED} but explicitly verify why your approach is sound.
        \end{itemize}

    \end{tcolorbox}
    \caption{Experience-guided problem-solving prompt for LiveCodeBench (Part 1 of 2).}
    \label{fig:prompt_livecodebench_with_reflection_1}
\end{figure*}

\begin{figure*}[ht]
    \centering
    \small
    \begin{tcolorbox}[
        colback=gray!5,
        colframe=black,
        boxrule=0.8pt,
        arc=2pt,
        left=8pt, right=8pt, top=8pt, bottom=8pt,
        title=\textbf{Experience-Guided Problem-Solving Prompt for LiveCodeBench (Part 2 of 2)}
    ]
        \ttfamily
        \scriptsize
        \setlength{\parskip}{0.5em}

        \textbf{3. Conflict Resolution \& Robustness:}
        \begin{itemize}[leftmargin=*, nosep]
            \item \textbf{Scenario:} You encounter a contradiction (e.g., two incompatible state definitions, a transition that breaks an invariant, or an approach whose complexity conflicts with the stated constraints).
            \item \textbf{Constraint:} Do NOT simply choose the "simpler-looking" or "more familiar" approach.
            \item \textbf{Action:} A contradiction usually means a \textbf{foundational assumption} (e.g., state meaning, greedy choice, graph/model construction, or boundary handling) is incorrect. \textbf{Backtrack to the very beginning}, re-read the problem statement, and challenge your initial setup.
        \end{itemize}

        \textbf{4. Simplification over Enumeration:}
        \begin{itemize}[leftmargin=*, nosep]
            \item \textbf{Guideline:} Before diving into fragile case-by-case implementation or excessive branching, pause and ask: "Is there a cleaner model, invariant, reduction, monotonicity property, or data-structure view that simplifies this?"
            \item \textbf{Goal:} Use the Reference Report to find these \textit{structural simplifications} rather than using it to complicate the problem with unnecessary cases or ad hoc fixes.
        \end{itemize}

        \vspace{0.5em}\noindent\textbf{Context from Previous Attempts:} \\
        \textcolor{red}{\{\{ experience\_context \}\}}

        \vspace{0.5em}\noindent\textbf{Instruction:} \\
        Reason step by step. Start by establishing the core constraints, the required algorithmic structure, and the implementation obligations. Consult the Reference Report critically: verify negative experiences before pruning, and use positive experiences only if they safely accelerate your work. Then produce a complete Python solution that matches the specification. Enclose your final code within triple backticks.

        \vspace{1em}\noindent\textbf{[User Message Template]}

        \#\#\# Question: \\
        \textcolor{red}{\{\{ question\_content \}\}}

        \vspace{0.5em}
        \textit{--- If starter code is provided: ---} \\
        \#\#\# Format: You will use the following starter code to write the solution to the problem and enclose your code within delimiters. \\
        \pyfence\\
        \textcolor{red}{\{\{ starter\_code \}\}}\\
        \codefence

        \vspace{0.5em}
        \textit{--- If no starter code is provided: ---} \\
        \#\#\# Format: Read the inputs from stdin solve the problem and write the answer to stdout (do not directly test on the sample inputs). Enclose your code within delimiters as follows. Ensure that when the python program runs, it reads the inputs, runs the algorithm and writes output to STDOUT. \\
        \pyfence\\
        \# YOUR CODE HERE\\
        \codefence

        \vspace{0.5em}
        \#\#\# Answer: (use the provided format with backticks)

    \end{tcolorbox}
    \caption{Experience-guided problem-solving prompt for LiveCodeBench (Part 2 of 2).}
    \label{fig:prompt_livecodebench_with_reflection_2}
\end{figure*}

\begin{figure*}[ht]
    \centering
    \begin{tcolorbox}[
        colback=gray!5,
        colframe=black,
        boxrule=0.8pt,
        arc=2pt,
        left=10pt, right=10pt, top=10pt, bottom=10pt,
        title=\textbf{PaCoRe Input Serialization Template for LiveCodeBench}
    ]
        \ttfamily
        \scriptsize
        \setlength{\parskip}{0.5em}
        
        You are given a problem and a list of reference responses. Your job is to analyze these references and provide your own response.
        
        \vspace{1em}
        
        Original Problem: \\
        \textcolor{red}{\{\{ original\_prompt \}\}}
        
        \vspace{1em}
        
        Reference Responses: \\
        \textcolor{red}{\{\% for response in ref\_responses \%\}} \\
        Reference \textcolor{red}{\{\{ loop.index \}\}}: \\
        \textcolor{red}{\{\{ response \}\}} \\
        \textcolor{red}{\{\% endfor \%\}}
        
        \vspace{1em}
        
        Now, based on the original problem and reference responses above, please provide your own comprehensive solution.
        
    \end{tcolorbox}
    \caption{PaCoRe input serialization template for LiveCodeBench}
    \label{fig:prompt_controller}
\end{figure*}


\begin{figure*}[ht]
    \centering
    \small
    \begin{tcolorbox}[
        colback=gray!5,
        colframe=black,
        boxrule=0.8pt,
        arc=2pt,
        left=8pt, right=8pt, top=8pt, bottom=8pt,
        title=\textbf{System Prompt for TravelPlanner.}
    ]
        \ttfamily
        \scriptsize
        \setlength{\parskip}{0.5em}
        
        You are a proficient planner. Based on the provided information and query, please give me a detailed plan, including specifics such as flight numbers (e.g., F0123456), restaurant names, and accommodation names. Note that all the information in your plan should be derived from the provided data. You must adhere to the format given in the example. Additionally, all details should align with commonsense. The symbol '-' indicates that information is unnecessary. For example, in the provided sample, you do not need to plan after returning to the departure city. When you travel to two cities in one day, you should note it in the 'Current City' section as in the example (i.e., from A to B).

        ***** Example ***** \\
        \textbf{Query:} Could you create a travel plan for 7 people from Ithaca to Charlotte spanning 3 days, from March 8th to March 14th, 2022, with a budget of \$30,200?

        \textbf{Travel Plan:} \\
        Day 1: \\
        Current City: from Ithaca to Charlotte \\
        Transportation: Flight Number: F3633413, from Ithaca to Charlotte, Departure Time: 05:38, Arrival Time: 07:46 \\
        Breakfast: Nagaland's Kitchen, Charlotte \\
        Attraction: The Charlotte Museum of History, Charlotte \\
        Lunch: Cafe Maple Street, Charlotte \\
        Dinner: Bombay Vada Pav, Charlotte \\
        Accommodation: Affordable Spacious Refurbished Room in Bushwick!, Charlotte

        Day 2: \\
        Current City: Charlotte \\
        Transportation: - \\
        Breakfast: Olive Tree Cafe, Charlotte \\
        Attraction: The Mint Museum, Charlotte;Romare Bearden Park, Charlotte. \\
        Lunch: Birbal Ji Dhaba, Charlotte \\
        Dinner: Pind Balluchi, Charlotte \\
        Accommodation: Affordable Spacious Refurbished Room in Bushwick!, Charlotte

        Day 3: \\
        Current City: from Charlotte to Ithaca \\
        Transportation: Flight Number: F3786167, from Charlotte to Ithaca, Departure Time: 21:42, Arrival Time: 23:26 \\
        Breakfast: Subway, Charlotte \\
        Attraction: Books Monument, Charlotte. \\
        Lunch: Olive Tree Cafe, Charlotte \\
        Dinner: Kylin Skybar, Charlotte \\
        Accommodation: -

        ***** Example Ends *****

        You MUST output the travel plan as a JSON array. Each element is a dict with keys: \\
        "days", "current\_city", "transportation", "breakfast", "attraction", "lunch", "dinner", "accommodation". \\
        Use "-" for fields that are not applicable. Attractions should be separated by semicolons with a trailing semicolon. \\
        Output ONLY the JSON array, no other text.

    \end{tcolorbox}
    \caption{System prompt for TravelPlanner.}
    \label{fig:prompt_travelplanner}
\end{figure*}


\begin{figure*}[ht]
    \centering
    \small
    \begin{tcolorbox}[
        colback=gray!5,
        colframe=black,
        boxrule=0.8pt,
        arc=2pt,
        left=8pt, right=8pt, top=8pt, bottom=8pt,
        title=\textbf{Experience Distillation Prompt for TravelPlanner (Part 1 of 4)}
    ]
        \ttfamily
        \scriptsize
        \setlength{\parskip}{0.5em}
        
        You are an expert travel planning analyst. You will be given a travel query, reference data, a generated plan, and evaluation results. Your job is to analyze every constraint violation and extract \textbf{generalizable, actionable rules} that a planner can follow on \textbf{future, unseen queries} — not just to fix this specific plan.

        \vspace{0.5em}\noindent\textbf{TravelPlanner Evaluation Criteria}

        A travel plan is evaluated on \textbf{two layers}. Layer 1 (Commonsense) must ALL pass before Layer 2 (Hard) is even checked.

        \vspace{0.5em}\noindent\textbf{Layer 1 — Commonsense Constraints (8 checks)}

        \textbf{1. is\_reasonable\_visiting\_city — Route \& City Validation} \\
        The evaluator builds a \texttt{city\_list} by scanning each day's \texttt{current\_city}:
        \begin{itemize}[leftmargin=*, nosep]
            \item If \texttt{current\_city} contains \texttt{"from"}, it extracts \textit{both} cities via regex \texttt{from A to B} and appends \texttt{[A, B]} to city\_list.
            \item Otherwise it appends the single city name.
        \end{itemize}
        Then it checks: \\
        (a) \textbf{Closed loop}: \texttt{city\_list[0] == city\_list[-1]} (must both be the origin city). \\
        (b) \textbf{Valid city sequence}: Every intermediate city (not first/last) must appear \textbf{consecutively at least 2 times} in city\_list. A city appearing only once in the middle is invalid. \\
        (c) \textbf{State match}: For trips > 3 days, every non-origin city must belong to the required destination state. \\
        (d) \textbf{Day-1 origin}: If Day 1 uses \texttt{"from A to B"}, then A must equal the origin city.

        \textbf{Implication for \texttt{current\_city} format}: On travel days (Day 1 and return day), you MUST use \texttt{"from OriginCity to DestCity"} so that both cities enter the city\_list. On non-travel days, use a plain city name. This is the ONLY way to form a closed loop while satisfying the consecutive-appearance rule.

        \textbf{2. is\_valid\_restaurants — No Restaurant Repetition} \\
        Every restaurant name (across ALL meals on ALL days) must be unique. The evaluator collects breakfast/lunch/dinner values across all days into a single list and fails if any name appears twice.

        \textbf{3. is\_valid\_attractions — No Attraction Repetition} \\
        Every attraction (split by \texttt{;}) across all days must be unique. Trailing semicolons are required: the evaluator splits by \texttt{;} and discards the last empty element.

        \textbf{4. is\_valid\_accommodation — Minimum Nights \& Presence}
        \begin{itemize}[leftmargin=*, nosep]
            \item Accommodation must be present (non-empty, non-\texttt{"-"}) for every day EXCEPT the return day.
            \item The evaluator groups \textbf{consecutive identical} accommodation entries and checks: \texttt{consecutive\_count >= minimum\_nights} from the database. If you use the same accommodation for fewer consecutive days than its minimum\_nights, this check FAILS.
        \end{itemize}

        \textbf{5. is\_valid\_transportation — Day-1 Required \& No Conflicts}
        \begin{itemize}[leftmargin=*, nosep]
            \item Day 1 MUST have transportation (non-empty, non-\texttt{"-"}).
            \item No mixing of conflicting modes in the entire trip: Self-driving + Flight is forbidden; Taxi + Self-driving is forbidden. Taxi + Flight is allowed.
        \end{itemize}

        \textbf{6. is\_valid\_information\_in\_current\_city — City-Tag Matching} \\
        The evaluator extracts the city(ies) from \texttt{current\_city} (either one or two if \texttt{"from A to B"}), then checks that each restaurant/attraction/accommodation/transportation string on that day \textbf{contains} at least one of those city names as a substring. Format must be \texttt{"ItemName, CityName"} (comma-separated) so the city name appears in the string.

        \textbf{7. is\_valid\_information\_in\_sandbox — Database Existence} \\
        Every item mentioned in the plan must actually exist in the reference database for the specified city:
        \begin{itemize}[leftmargin=*, nosep]
            \item Restaurants: name must match (substring) in the restaurants DB for that city
            \item Attractions: name must match in the attractions DB for that city
            \item Accommodations: name must match in the accommodations DB for that city
            \item Flights: flight number + origin + destination must all match
            \item Taxi/Self-driving: the city pair must have a valid route in the distance matrix
        \end{itemize}

        \textbf{8. is\_not\_absent — Completeness} \\
        Every day must have ALL 7 fields non-empty: \texttt{current\_city}, \texttt{transportation}, \texttt{breakfast}, \texttt{attraction}, \texttt{lunch}, \texttt{dinner}, \texttt{accommodation}. The ONLY exception: the return day's \texttt{accommodation} may be \texttt{"-"}. All other fields on all days (including the return day's breakfast/lunch/dinner/attraction) should be filled with valid data, NOT \texttt{"-"}.

    \end{tcolorbox}
    \caption{Experience distillation prompt for TravelPlanner (Part 1 of 4).}
    \label{fig:prompt_travelplanner_reflection_1}
\end{figure*}

\begin{figure*}[ht]
    \centering
    \small
    \begin{tcolorbox}[
        colback=gray!5,
        colframe=black,
        boxrule=0.8pt,
        arc=2pt,
        left=8pt, right=8pt, top=8pt, bottom=8pt,
        title=\textbf{Experience Distillation Prompt for TravelPlanner (Part 2 of 4)}
    ]
        \ttfamily
        \scriptsize
        \setlength{\parskip}{0.5em}
        
        \vspace{0.5em}\noindent\textbf{Layer 2 — Hard Constraints (5 checks, only evaluated if \texttt{is\_not\_absent} AND \texttt{is\_valid\_information\_in\_sandbox} both pass)}

        \textbf{1. valid\_cost — Budget Compliance} \\
        Total trip cost must $\le$ budget. The evaluator calculates cost as:
\begin{verbatim}
For each day:
  + flight_price × N_people                              (if flight)
  + taxi_cost × ceil(N_people / 4)                       (if taxi; capacity = 4)
  + self_driving_cost × ceil(N_people / 5)               (if self-driving; capacity = 5)
  + breakfast_avg_cost × N_people                         (if breakfast exists)
  + lunch_avg_cost × N_people                             (if lunch exists)
  + dinner_avg_cost × N_people                            (if dinner exists)
  + accommodation_price × ceil(N_people / max_occupancy)  (if accommodation exists)
\end{verbatim}
        \textbf{CRITICAL}: Taxi capacity is 4, self-driving capacity is 5. These are hardcoded in the evaluator. The \texttt{cost} value from the distance matrix is per-vehicle, and the evaluator multiplies by \texttt{ceil(N\_people / capacity)}.

        \textbf{2. valid\_room\_rule — House Rule Compliance} \\
        If the query specifies a house rule constraint (e.g., \texttt{"children under 10"}), every accommodation's \texttt{house\_rules} field must NOT contain \texttt{"No children under 10"}. The mapping:
        \begin{itemize}[leftmargin=*, nosep]
            \item constraint \texttt{"smoking"} $\rightarrow$ must not have \texttt{"No smoking"}
            \item constraint \texttt{"pets"} $\rightarrow$ must not have \texttt{"No pets"}
            \item constraint \texttt{"parties"} $\rightarrow$ must not have \texttt{"No parties"}
            \item constraint \texttt{"children under 10"} $\rightarrow$ must not have \texttt{"No children under 10"}
            \item constraint \texttt{"visitors"} $\rightarrow$ must not have \texttt{"No visitors"}
        \end{itemize}

        \textbf{3. valid\_cuisine — Cuisine Coverage} \\
        If cuisine types are specified, at least one meal (in a \textbf{destination city}, not the origin city) must serve each required cuisine. The evaluator checks the \texttt{Cuisines} field in the restaurant DB.

        \textbf{4. valid\_transportation — Mode Restriction} \\
        If \texttt{"no flight"} $\rightarrow$ no flights allowed anywhere. If \texttt{"no self-driving"} $\rightarrow$ no self-driving allowed anywhere.

        \textbf{5. valid\_room\_type — Room Type Matching} \\
        If room type is specified, ALL accommodations must match:
        \begin{itemize}[leftmargin=*, nosep]
            \item \texttt{"entire room"} $\rightarrow$ room type must be \texttt{"Entire home/apt"} in the DB
            \item \texttt{"private room"} $\rightarrow$ room type must be \texttt{"Private room"}
            \item \texttt{"shared room"} $\rightarrow$ room type must be \texttt{"Shared room"}
            \item \texttt{"not shared room"} $\rightarrow$ room type must NOT be \texttt{"Shared room"}
        \end{itemize}

    \end{tcolorbox}
    \caption{Experience distillation prompt for TravelPlanner (Part 2 of 4).}
    \label{fig:prompt_travelplanner_reflection_2}
\end{figure*}

\begin{figure*}[ht]
    \centering
    \small
    \begin{tcolorbox}[
        colback=gray!5,
        colframe=black,
        boxrule=0.8pt,
        arc=2pt,
        left=8pt, right=8pt, top=8pt, bottom=8pt,
        title=\textbf{Experience Distillation Prompt for TravelPlanner (Part 3 of 4)}
    ]
        \ttfamily
        \scriptsize
        \setlength{\parskip}{0.5em}
        
        \vspace{0.5em}\noindent\textbf{Your Task}

        Analyze the evaluation failures and produce \textbf{two types of reusable knowledge}:

        \textbf{1. Positive Experiences — Actionable rules for future planning} \\
        Each positive experience should be a \textbf{general rule} that applies beyond this specific query. Focus on:
        \begin{itemize}[leftmargin=*, nosep]
            \item \textbf{Format rules}: Exact \texttt{current\_city} format, item naming conventions (\texttt{"Name, City"}), attraction separator (\texttt{;})
            \item \textbf{Cost formulas}: Precise per-category cost calculation with correct capacity divisors (taxi=4, self-driving=5)
            \item \textbf{Data cross-referencing}: How to verify items exist in the reference data for the correct city
            \item \textbf{Structural rules}: Closed-loop routing, consecutive city-stay requirements, return-day field handling
            \item \textbf{Constraint satisfaction}: Room rules, cuisine coverage, min-nights, budget verification
        \end{itemize}

        \textbf{2. Negative Experiences — Specific mistakes to avoid} \\
        Each negative experience should describe a \textbf{concrete error pattern} (not just "be careful"), e.g.:
        \begin{itemize}[leftmargin=*, nosep]
            \item "Writing \texttt{current\_city} as a plain city name on travel days instead of \texttt{from A to B}, which breaks the closed-loop check"
            \item "Assuming taxi can carry 6 people (capacity is 4), causing cost underestimation by a factor of 2"
            \item "Setting meals to \texttt{-} on the return day, which triggers is\_not\_absent failure"
        \end{itemize}

        \vspace{0.5em}\noindent\textbf{Output Format}

\begin{verbatim}
{
    "positive experiences": [
        {
            "statement": "A generalizable, actionable rule (imperative mood)",
            "category": "city_route|restaurant|attraction|...|data_validity",
            "severity": "critical|important|minor",
            "eval_check": "the specific evaluation check this addresses"
        }
    ],
    "negative experiences": [
        {
            "statement": "A concrete error pattern to avoid",
            "category": "city_route|restaurant|attraction|...|data_validity",
            "eval_check": "the specific evaluation check this addresses"
        }
    ],
    "summary": "Brief overall diagnosis: which checks failed, root causes..."
}
\end{verbatim}

        \vspace{0.5em}\noindent\textbf{Guidelines}
        \begin{itemize}[leftmargin=*, nosep]
            \item Generate 3--8 positive experiences and 2--5 negative experiences, focusing on the FAILED checks.
            \item Positive Experiences should be \textbf{imperative} ("Always verify...", "Calculate...", "Ensure...") and \textbf{self-contained} — a planner reading only the positive experience should know exactly what to do.
            \item Reference concrete data patterns from this failure (e.g., actual budget numbers, specific constraint values) as \textbf{examples}, but phrase the rule generally.
            \item Do NOT generate vague rules like "be careful with budget" — instead: "Calculate taxi cost as: per\_vehicle\_cost $\times$ $\lceil$N\_people / 4$\rceil$ (taxi capacity is always 4), then sum across all travel days."
            \item If a check PASSED, you do not need to generate rules for it unless you see a near-miss worth noting.
        \end{itemize}

        Output ONLY the JSON object, no other text.

    \end{tcolorbox}
    \caption{Experience distillation prompt for TravelPlanner (Part 3 of 4).}
    \label{fig:prompt_travelplanner_reflection_3}
\end{figure*}

\begin{figure*}[ht]
    \centering
    \small
    \begin{tcolorbox}[
        colback=gray!5,
        colframe=black,
        boxrule=0.8pt,
        arc=2pt,
        left=8pt, right=8pt, top=8pt, bottom=8pt,
        title=\textbf{Experience Distillation Prompt for TravelPlanner (Part 4 of 4)}
    ]
        \ttfamily
        \scriptsize
        \setlength{\parskip}{0.5em}
        
        \textbf{\#\# Query} \\
        \textcolor{red}{\{step1\_result['query']\}}

        \textbf{\#\# Query Details}
        \begin{itemize}[leftmargin=*, nosep]
            \item Origin city: \textcolor{red}{\{step1\_result['query\_data']['org']\}}
            \item Destination state: \textcolor{red}{\{step1\_result['query\_data']['dest']\}}
            \item Number of days: \textcolor{red}{\{step1\_result['query\_data']['days']\}}
            \item Number of people: \textcolor{red}{\{step1\_result['query\_data']['people\_number']\}}
            \item Budget: \$\textcolor{red}{\{step1\_result['query\_data']['budget']\}}
            \item Number of cities to visit: \textcolor{red}{\{step1\_result['query\_data']['visiting\_city\_number']\}}
            \item Date range: \textcolor{red}{\{step1\_result['query\_data']['date']\}}
            \item Local constraints: \\
            \textcolor{red}{\{constraint\_text\}}
        \end{itemize}

        \textbf{\#\# Reference Information (provided to the planner)} \\
        \textcolor{red}{\{ref\_text\}}

        \textbf{\#\# Generated Plan} \\
        \textcolor{red}{\{plan\_text\}}

        Based on the evaluation results above, analyze each FAILED check, identify the root cause in the generated plan, and extract generalizable positive experiences and negative experiences. Output ONLY the JSON object.

    \end{tcolorbox}
    \caption{Experience distillation prompt for TravelPlanner (Part 4 of 4).}
    \label{fig:prompt_travelplanner_reflection_4}
\end{figure*}


\begin{figure*}[ht]
    \centering
    \small
    \begin{tcolorbox}[
        colback=gray!5,
        colframe=black,
        boxrule=0.8pt,
        arc=2pt,
        left=8pt, right=8pt, top=8pt, bottom=8pt,
        title=\textbf{Experience-Guided Problem-Solving Prompt for TravelPlanner (Part 1 of 2)}
    ]
        \ttfamily
        \scriptsize
        \setlength{\parskip}{0.5em}
        
        You are a proficient planner. Based on the provided information, query, and past experiences/lessons learned, please give me a detailed plan, including specifics such as flight numbers (e.g., F0123456), restaurant names, and accommodation names. Note that all the information in your plan should be derived from the provided data. You must adhere to the format given in the example. Additionally, all details should align with commonsense. The symbol '-' indicates that information is unnecessary. For example, in the provided sample, you do not need to plan after returning to the departure city. When you travel to two cities in one day, you should note it in the 'Current City' section as in the example (i.e., from A to B).

        ***** Example ***** \\
        \textbf{Query:} Could you create a travel plan for 7 people from Ithaca to Charlotte spanning 3 days, from March 8th to March 14th, 2022, with a budget of \$30,200?

        \textbf{Travel Plan:} \\
        Day 1: \\
        Current City: from Ithaca to Charlotte \\
        Transportation: Flight Number: F3633413, from Ithaca to Charlotte, Departure Time: 05:38, Arrival Time: 07:46 \\
        Breakfast: Nagaland's Kitchen, Charlotte \\
        Attraction: The Charlotte Museum of History, Charlotte; \\
        Lunch: Cafe Maple Street, Charlotte \\
        Dinner: Bombay Vada Pav, Charlotte \\
        Accommodation: Affordable Spacious Refurbished Room in Bushwick!, Charlotte

        Day 2: \\
        Current City: Charlotte \\
        Transportation: - \\
        Breakfast: Olive Tree Cafe, Charlotte \\
        Attraction: The Mint Museum, Charlotte;Romare Bearden Park, Charlotte; \\
        Lunch: Birbal Ji Dhaba, Charlotte \\
        Dinner: Pind Balluchi, Charlotte \\
        Accommodation: Affordable Spacious Refurbished Room in Bushwick!, Charlotte

        Day 3: \\
        Current City: from Charlotte to Ithaca \\
        Transportation: Flight Number: F3786167, from Charlotte to Ithaca, Departure Time: 21:42, Arrival Time: 23:26 \\
        Breakfast: Subway, Charlotte \\
        Attraction: Books Monument, Charlotte; \\
        Lunch: Taste of Beijing, Charlotte \\
        Dinner: Kylin Skybar, Charlotte \\
        Accommodation: -

        ***** Example Ends *****

        You MUST output the travel plan as a JSON array. Each element is a dict with keys: \\
        "days", "current\_city", "transportation", "breakfast", "attraction", "lunch", "dinner", "accommodation". \\
        Use "-" for fields that are not applicable. Attractions should be separated by semicolons with a trailing semicolon. \\
        Output ONLY the JSON array, no other text.

    \end{tcolorbox}
    \caption{Experience-guided problem-solving prompt for TravelPlanner (Part 1 of 2).}
    \label{fig:prompt_travelplanner_with_reflection_1}
\end{figure*}

\begin{figure*}[ht]
    \centering
    \small
    \begin{tcolorbox}[
        colback=gray!5,
        colframe=black,
        boxrule=0.8pt,
        arc=2pt,
        left=8pt, right=8pt, top=8pt, bottom=8pt,
        title=\textbf{Experience-Guided Problem-Solving Prompt for TravelPlanner (Part 2 of 2)}
    ]
        \ttfamily
        \scriptsize
        \setlength{\parskip}{0.5em}
        
        \vspace{0.5em}\noindent\textbf{\#\# CRITICAL RULES (violations cause automatic failure)}
        \begin{enumerate}[leftmargin=*, nosep]
            \item \textbf{current\_city format}: On travel days (departure and return), use \texttt{"from CityA to CityB"}. On non-travel days, use a plain city name. The trip MUST form a closed loop: the first city in Day 1 must equal the last city on the return day (both = origin city). Every intermediate city must appear on at least 2 consecutive days.
            \item \textbf{No restaurant repetition}: Every restaurant name must appear EXACTLY ONCE across all meals on all days. Never reuse a restaurant even across different meal types.
            \item \textbf{No attraction repetition}: Every attraction name must appear exactly once across all days. Use semicolons to separate multiple attractions, with a trailing semicolon.
            \item \textbf{Return day completeness}: On the return day, set accommodation to \texttt{"-"}, but you MUST still fill in breakfast, lunch, dinner, attraction, and transportation with valid data from the last destination city — do NOT set them to \texttt{"-"}.
            \item \textbf{Item format}: Always format restaurants, attractions, and accommodations as \texttt{"Name, City"} (name followed by comma and city name). Every item must exist in the provided reference data for that city.
            \item \textbf{Accommodation min-nights}: If you use the same accommodation for N consecutive days, the accommodation's \texttt{minimum nights} in the reference data must be $\le$ N.
            \item \textbf{Cost calculation}: Budget check uses these exact formulas:
            \begin{itemize}[leftmargin=*, nosep]
                \item Flights: price $\times$ N\_people
                \item Taxi: cost\_per\_vehicle $\times$ \texttt{ceil(N\_people / 4)} (taxi capacity = 4)
                \item Self-driving: cost\_per\_vehicle $\times$ \texttt{ceil(N\_people / 5)} (capacity = 5)
                \item Meals: average\_cost $\times$ N\_people
                \item Accommodation: price\_per\_night $\times$ \texttt{ceil(N\_people / max\_occupancy)}
            \end{itemize}
            Verify total $\le$ budget before finalizing.
            \item \textbf{Room constraints}: If house rule constraint is given (e.g., "children under 10"), the accommodation must NOT have "No children under 10" in its house\_rules. If room type is "entire room", the accommodation's room type must be "Entire home/apt" in the DB.
        \end{enumerate}

        \vspace{0.5em}\noindent\textbf{\#\# Past Experiences / Lessons Learned}

        Pay special attention to the following experiences from past attempts on THIS specific query:

        \textcolor{red}{\{\{ experiences \}\}}

        \vspace{0.5em}\noindent\textbf{Given Information:} \\
        \textcolor{red}{\{\{ reference\_info \}\}}

        \vspace{0.5em}\noindent\textbf{Query:} \\
        \textcolor{red}{\{\{ query \}\}}

    \end{tcolorbox}
    \caption{Experience-guided problem-solving prompt for TravelPlanner (Part 2 of 2).}
    \label{fig:prompt_travelplanner_with_reflection_2}
\end{figure*}

\begin{figure*}[ht]
    \centering
    \small
    \begin{tcolorbox}[
        colback=gray!5,
        colframe=black,
        boxrule=0.8pt,
        arc=2pt,
        left=8pt, right=8pt, top=8pt, bottom=8pt,
        title=\textbf{PaCoRe Input Serialization Template for TravelPlanner (Part 1 of 2)}
    ]
        \ttfamily
        \scriptsize
        \setlength{\parskip}{0.5em}
        
        You are a proficient travel planner.

        \vspace{0.5em}\noindent\textbf{Domain Knowledge:} \\
        When creating a travel plan, you must satisfy ALL of the following constraints:
        \begin{itemize}[leftmargin=*, nosep]
            \item \textbf{Budget}: The total cost (flights $\times$ people + meals $\times$ people + accommodation $\times$ $\lceil$people/rooms$\rceil$ + ground transport) must NOT exceed the given budget.
            \item \textbf{Room Rule}: If "not shared room", each person must have their own room.
            \item \textbf{Cuisine}: If specific cuisine types are requested (e.g., Chinese, Italian), ALL meals (breakfast, lunch, dinner) must match.
            \item \textbf{Room Type}: If a room type is requested (e.g., entire room, private room), ALL accommodations must match.
            \item \textbf{Transportation}: If specific transportation is requested (e.g., no flight, self-driving only), ALL transportation must comply.
            \item \textbf{City Validity}: All restaurants, attractions, and accommodations must actually exist in the city being visited.
            \item \textbf{Completeness}: Every field must be filled for each day. Use "-" only where explicitly appropriate (e.g., accommodation on the last day when returning home).
            \item \textbf{Visiting Cities}: The number of cities visited must match the query requirement.
            \item \textbf{Valid Info}: All flight numbers, restaurant names, accommodation names must come from the provided reference data.
        \end{itemize}

        \vspace{0.5em}\noindent\textbf{Given Information:} \\
        \textcolor{red}{\{\{ reference\_info \}\}}

        \vspace{0.5em}\noindent\textbf{Query:} \\
        \textcolor{red}{\{\{ query \}\}}

        \vspace{0.5em}\noindent\textbf{Reference Plans from Previous Attempts:} \\
        \textcolor{blue}{\{\% for plan\_text in ref\_plans \%\}} \\
        --- Reference \textcolor{red}{\{\{ loop.index \}\}} --- \\
        \textcolor{red}{\{\{ plan\_text \}\}} \\
        \textcolor{blue}{\{\% endfor \%\}}

        \vspace{0.5em}\noindent\textbf{Instructions:} \\
        You are given the query, reference information, and \textcolor{red}{\{\{ ref\_plans|length \}\}} reference plan(s) from previous attempts. These references may contain errors (wrong flights, budget violations, invalid restaurants, missing days, constraint violations, etc.).

        Your task:
        \begin{enumerate}[leftmargin=*, nosep]
            \item Analyze the reference plans critically — identify which parts are correct and which have errors.
            \item Cross-check every item (flights, restaurants, accommodations, attractions) against the provided reference data.
            \item Verify budget constraints: calculate total cost and ensure it does not exceed the budget.
            \item Verify all other constraints (room rule, cuisine, transportation, city validity, etc.).
            \item Provide your own correct and complete plan.
        \end{enumerate}

        Do NOT blindly copy reference plans. Verify each item independently against the provided data.

    \end{tcolorbox}
    \caption{PaCoRe  input serialization template for TravelPlanner (Part 1 of 2).}
    \label{fig:prompt_travelplanner_controller_1}
\end{figure*}

\begin{figure*}[ht]
    \centering
    \small
    \begin{tcolorbox}[
        colback=gray!5,
        colframe=black,
        boxrule=0.8pt,
        arc=2pt,
        left=8pt, right=8pt, top=8pt, bottom=8pt,
        title=\textbf{PaCoRe Input Serialization Template for TravelPlanner (Part 2 of 2)}
    ]
        \ttfamily
        \scriptsize
        \setlength{\parskip}{0.5em}

        ***** Example ***** \\
        \textbf{Query:} Could you create a travel plan for 7 people from Ithaca to Charlotte spanning 3 days, from March 8th to March 14th, 2022, with a budget of \$30,200?

        \textbf{Travel Plan:} \\
        Day 1: \\
        Current City: from Ithaca to Charlotte \\
        Transportation: Flight Number: F3633413, from Ithaca to Charlotte, Departure Time: 05:38, Arrival Time: 07:46 \\
        Breakfast: Nagaland's Kitchen, Charlotte \\
        Attraction: The Charlotte Museum of History, Charlotte; \\
        Lunch: Cafe Maple Street, Charlotte \\
        Dinner: Bombay Vada Pav, Charlotte \\
        Accommodation: Affordable Spacious Refurbished Room in Bushwick!, Charlotte

        Day 2: \\
        Current City: Charlotte \\
        Transportation: - \\
        Breakfast: Olive Tree Cafe, Charlotte \\
        Attraction: The Mint Museum, Charlotte;Romare Bearden Park, Charlotte; \\
        Lunch: Birbal Ji Dhaba, Charlotte \\
        Dinner: Pind Balluchi, Charlotte \\
        Accommodation: Affordable Spacious Refurbished Room in Bushwick!, Charlotte

        Day 3: \\
        Current City: from Charlotte to Ithaca \\
        Transportation: Flight Number: F3786167, from Charlotte to Ithaca, Departure Time: 21:42, Arrival Time: 23:26 \\
        Breakfast: Subway, Charlotte \\
        Attraction: Books Monument, Charlotte; \\
        Lunch: Taste of Beijing, Charlotte \\
        Dinner: Kylin Skybar, Charlotte \\
        Accommodation: -

        ***** Example Ends *****

        You MUST output the travel plan as a JSON array. Each element is a dict with keys: \\
        "days", "current\_city", "transportation", "breakfast", "attraction", "lunch", "dinner", "accommodation". \\
        Use "-" for fields that are not applicable. Attractions should be separated by semicolons with a trailing semicolon. \\
        Output ONLY the JSON array, no other text.

    \end{tcolorbox}
    \caption{PaCoRe  input serialization template for TravelPlanner (Part 2 of 2).}
    \label{fig:prompt_travelplanner_controller_2}
\end{figure*}


\end{document}